\documentclass[11pt,fullpage, letterpaper,twoside]{article}
\usepackage{helvet}
\usepackage{epsfig} 
\usepackage{subfigure}
\usepackage{cite}

\usepackage{algorithm}
\usepackage{algorithmic}
\usepackage{amsmath}
\usepackage{amsthm}
\usepackage{amssymb}

\usepackage{dsfont}

\usepackage{multirow}
\usepackage{booktabs}
\usepackage{hhline}
\usepackage{multirow}
\usepackage{url}
\usepackage[margin=1in]{geometry}

\newtheorem{theorem}{Theorem}
\newtheorem{lemma}{Lemma}
\newtheorem{assumption}{Assumption}
\newtheorem{corollary}{Corollary}
\newtheorem{remark}{Remark}
\newcommand{\1}[1]{\mathds{1}_{\left[#1\right]}}

\usepackage{CJK}
\usepackage{indentfirst}
\usepackage[colorlinks,linkcolor=red,anchorcolor=blue,citecolor=green]{hyperref}
\usepackage{fancyhdr}
\begin{document}

\title{Loss-Sensitive Generative Adversarial Networks on Lipschitz Densities
}


\author{Guo-Jun Qi \\
Laboratory for {\bf MA}chine {\bf P}erception and {\bf LE}arning \\
\url{http://maple.cs.ucf.edu}\\
University of Central Florida\\
\tt\small guojun.qi@ucf.edu          
}



\date{}

\maketitle

\begin{abstract}

In this paper, we present the Lipschitz regularization theory and algorithms for a novel Loss-Sensitive Generative Adversarial Network (LS-GAN). Specifically, it trains a loss function to distinguish between real and fake samples by designated margins, while learning a generator alternately to produce realistic samples by minimizing their losses.
The LS-GAN further regularizes its loss function with a Lipschitz regularity condition on the density of real data, yielding a regularized model that can better generalize to produce new data from a reasonable number of training examples than the classic GAN.
We will further present a Generalized LS-GAN (GLS-GAN) and show it contains a large family of regularized GAN models, including both LS-GAN and Wasserstein GAN, as its special cases.  Compared with the other GAN models, we will conduct experiments to show both LS-GAN and GLS-GAN exhibit competitive ability in generating new images in terms of the Minimum Reconstruction Error (MRE) assessed on a separate test set.
We further extend the LS-GAN to a conditional form for supervised and semi-supervised learning problems, and demonstrate its outstanding performance on image classification tasks.

\end{abstract}


\maketitle


%

\section{Introduction}

A classic Generative Adversarial Net (GAN) \cite{goodfellow2014generative} learns a discriminator and a generator by playing a two-player minimax game to generate samples from a data distribution. The discriminator is trained to distinguish real samples from those generated by the generator, and it in turn guides the generator to produce realistic samples that can fool the discriminator.



However, from both theoretical and practical perspectives, a critical question is {\em whether the GAN can generate realistic samples from arbitrary data distribution without any prior? If not, what kind of prior ought to be imposed on the data distribution to regularize the GAN?} Indeed,
the classic GAN \cite{goodfellow2014generative} imposes no prior on the data distribution.  This represents an ambitious goal to
generate samples from any distributions. However, it in turn requires a non-parametric discriminator to prove the distributional consistency between generated and real samples by assuming the model has infinite capacity (see Section 4 of \cite{goodfellow2014generative}).

This is a too strong assumption to establish the theoretical basis for the GAN. Moreover, with such an assumption, its generalizability becomes susceptible.
Specifically, one could argue the learned generator may be overfit by an unregularized discriminator in an non-parametric fashion by merely memorizing or interpolating training examples. In other words, it could lack the generalization ability to generate new samples out of existing data. 
Indeed, Arora et al. \cite{arora2017generalization} have shown that the GAN minimizing the Jensen-Shannon distance between the distributions of generated and real data could fail to generalize to produce new samples with a reasonable size of training set.
Thus, a properly regularized GAN is demanded to establish provable generalizability by focusing on a restricted yet still sufficiently large family of data distributions.

\subsection{Objective: Towards Regularized GANs}
In this paper, we attempt to develop regularization theory and algorithms for a novel Loss-Sensitive GAN (LS-GAN).
Specifically, we introduce a loss function to quantify the quality of generated samples. A constraint is imposed so that the loss of a real sample should be smaller than that of a generated counterpart. Specifically, in the learning algorithm, we will define margins to separate the losses between generated and real samples. Then, an optimal generator will be trained to produce realistic samples with minimum losses. The loss function and the generator will be trained in an adversarial fashion until generated samples become indistinguishable from real ones.

We will also develop new theory to analyze the LS-GAN on the basis of Lipschitz regularity.  We note that the reason of making non-parametric assumption of infinite capacity on the discriminator in the classic GAN is due to its ambitious goal to generate data from any arbitrary distribution.  However,
no free lunch \cite{wolpert1996lack} principle reminds us of the need to impose a suitable prior on the data distribution from which real samples are generated.  This inspires us to impose a
Lipschitz regularity condition by assuming the data density does not change abruptly.
Based on this mild condition, we will show that the density of generated samples by LS-GAN can exactly match that of real data.

More importantly, the Lipschitz regularity allows us to prove the LS-GAN can well generalize to produce new data from training examples.
To this end, we will provide a Probably Approximate Correct (PAC)-style theorem by showing the empirical LS-GAN model trained with a reasonable number of examples can be sufficiently close to the oracle LS-GAN trained with hypothetically known data distribution, thereby proving the generalizability of LS-GAN in generating samples from any Lipschitz data distribution.

We will also make a non-parametric analysis of the LS-GAN. It does not rely on any parametric form of the loss function to characterize its optimality in the space of Lipschtiz functions. It gives both the upper and lower bounds of the optimal loss, which are cone-shaped with non-vanishing gradient. This suggests that the LS-GAN can provide sufficient gradient to update its LS-GAN generator even if the loss function has been fully optimized, thus avoiding the vanishing gradient problem that could occur in training the GAN \cite{arjovsky2017towards}.

\subsection{Extensions: Generalized and Conditional LS-GANs}

We further present a generalized form of LS-GAN (GLS-GAN) and conduct experiment to demonstrate it has the best generalization ability. We will show this is not a surprising result as the GLS-GAN contains a large family of regularized GANs with both LS-GAN and Wasserstein GAN (WGAN) \cite{wgan17} as its special cases.
Moreover, we will extend a Conditional LS-GAN (CLS-GAN) that can generate samples from given conditions.  In particular, with class labels being conditions, the learned loss function can be used as a classifier for both supervised and semi-supervised learning.  The advantage of such a classifier arises from its ability of exploring generated examples to uncover intrinsic variations for different classes. Experiment results demonstrate competitive performance of the CLS-GAN classifier compared with the state-of-the-art models.

\vspace{-2mm}
\subsection{Paper Structure}
The remainder of this paper is organized as follows. Section~\ref{sec:related} reviews the related work, and the proposed LS-GAN is presented in Section~\ref{sec:lsgan}.  In Section~\ref{sec:theory}, we will analyze the LS-GAN by proving the distributional consistency between generated and real data with the Lipschitz regularity condition on the data distribution. In Section~\ref{sec:alg}, we will discuss the generalizability problem arising from using sample means to approximate the expectations in the training objectives. We will make a comparison with Wasserstein GAN (WGAN) in Section~\ref{sec:wgan}, and present a generalized LS-GAN with both WGAN and LS-GAN as its special cases in Section~\ref{sec:glsgan}. 
A non-parametric analysis of the algorithm is followed in Section~\ref{sec:nonparam}.
Then we will show how the model can be extended to a conditional model for both supervised and semi-supervised learning in Section~\ref{sec:clsgan}. Experiment results are presented in Section~\ref{sec:exp}, and we conclude in Section~\ref{sec:concl}.

\noindent{\bf Source codes.} The source codes for both LS-GAN and GLS-GAN are available at \url{https://github.com/maple-research-lab}, in the frameworks of torch, pytorch and tensorflow. LS-GAN is also supported by Microsoft CNTK at \url{https://www.cntk.ai/pythondocs/CNTK_206C_WGAN_LSGAN.html}.



\section{Related Work}\label{sec:related}

Deep generative models, especially the Generative Adversarial Net (GAN) \cite{goodfellow2014generative}, have attracted many attentions recently due to their demonstrated abilities of generating real samples following the underlying data densities.  In particular, the GAN attempts to learn a pair of discriminator and generator by playing a maximin game to seek an equilibrium, in which the discriminator is trained by distinguishing real samples from generated ones and the generator is optimized to produce samples that can fool the discriminator.

A family of GAN architectures have been proposed to implement this idea.  For example, recent progresses \cite{radford2015unsupervised,salimans2016improved} have shown impressive performances on synthesizing photo-realistic images by constructing multiple strided and factional-strided convolutional layers for discriminators and generators.  On the contrary, \cite{denton2015deep} proposed to use a Laplacian pyramid to produce high-quality images by iteratively adding multiple layers of noises at different resolutions. \cite{im2016generating} presented to train a recurrent generative model by using adversarial training to unroll gradient-based optimizations to create high quality images.

In addition to designing different GAN networks, research efforts have been made to train the GAN by different criteria. For example,
\cite{zhao2016energy} presented an energy-based GAN by minimizing an energy function to learn an optimal discriminator, and an auto-encoder structured discriminator is presented to compute the energy.  The authors also present a theoretical analysis by showing this variant of GAN can generate samples whose density can recover the underlying true data density.  However, it still needs to assume the discriminator has infinite modeling capacity to prove the result in a non-parametric fashion, and its generalizability of producing new data out of training examples is unknown without theoretical proof or empirical evidence.
In addition, \cite{nowozin2016f} presented to analyze the GAN from information theoretical perspective, and they seek to minimize the variational estimate of f-divergence, and show that the classic GAN is included as a special case of f-GAN. In contrast, InfoGAN \cite{chen2016infogan} proposed another information-theoretic GAN to learn disentangled representations capturing various latent concepts and factors in generating samples. Most recently, \cite{wgan17} propose to minimize the Earth-Mover distance between the density of generated samples and the true data density, and they show the resultant Wasserstein GAN (WGAN) can address the vanishing gradient problem that the classic GAN suffers.

Besides the class of GANs, there exist other models that also attempt to generate natural images. For example, \cite{gatys2015neural} rendered images by matching features in a convolutional network with respect to reference images.  \cite{dosovitskiy2015learning} used deconvolutional network to render 3D chair models in various styles and viewpoints.  \cite{gregor2015draw} introduced a deep recurrent neutral network architecture for image generation with a sequence of variational auto-encoders to iteratively construct complex images.

Recent efforts have also been made on leveraging the learned representations by deep generative networks to improve the classification accuracy when it is too difficult or expensive to label sufficient training examples.  For example, \cite{kingma2014semi} presented variational auto-encoders \cite{kingma2013auto} by combining deep generative models and approximate variational inference to explore both labeled and unlabeled data.
\cite{salimans2016improved} treated the samples from the GAN generator as a new class, and explore unlabeled examples by assigning them to a class different from the new one. \cite{rasmus2015semi} proposed to train a ladder network \cite{valpola2015neural} by minimizing the sum of supervised and unsupervised cost functions through back-propagation, which avoids the conventional layer-wise pre-training approach. \cite{springenberg2015unsupervised} presented an approach to learning a discriminative classifier by trading-off mutual information between observed examples and their predicted classes against an adversarial generative model. \cite{dumoulin2016adversarially} sought to jointly distinguish between not only real and generated samples but also their latent variables in an adversarial process. Recently, \cite{qi2017global} presented a novel paradigm of localized GANs to explore the local consistency of classifiers in local coordinate charts, as well as showed an intrinsic connection with Laplace-Beltrami operator along the manifold. These methods have shown promising results for classification tasks by leveraging deep generative models.


\section{Loss-Sensitive GAN}\label{sec:lsgan}
The classic GAN consists of two players -- a generator producing samples from random noises,
and a discriminator distinguishing real and fake samples. The generator and discriminator are trained in an adversarial fashion to reach an equilibrium in which generated samples become indistinguishable from their real counterparts.


On the contrary, in the LS-GAN we seek to learn a {\em loss function} $L_\theta(\mathbf x)$ parameterized with $\theta$ by assuming that a real example ought to have a smaller loss than a generated sample by a desired margin. Then the generator
can be trained to generate realistic samples by minimizing their losses.

Formally, consider a generator function $G_\phi$ that produces a sample  $G_\phi(\mathbf z)$ by transforming a noise input $\mathbf z\sim P_z(\mathbf z)$ drawn from a simple distribution $P_z$ such as uniform and Gaussian distributions.
Then for
a real example $\mathbf x$ and a generated sample $G_\phi(\mathbf z)$, the loss function can be trained to distinguish them with the following constraint:
\begin{equation}\label{eq:ls}
L_\theta(\mathbf x) \leq L_\theta(G_\phi(\mathbf z))- \Delta(\mathbf x,G_\phi(\mathbf z))
\end{equation}
where $\Delta(\mathbf x,G_\phi(\mathbf z))$ is the margin measuring the difference between $\mathbf x$ and $G_\phi(\mathbf z)$. This constraint requires a real sample be separated from a generated counterpart in terms of their losses by at least a margin of $\Delta(\mathbf x,G_\phi(\mathbf z))$.

The above hard constraint can be relaxed by introducing a nonnegative slack variable $\xi_{\mathbf x,\mathbf z}$
that quantifies the violation of the above constraint. This results in the following minimization problem to learn the loss function $L_\theta$ given a fixed generator $G_{\phi^*}$,
\begin{align}\label{eq:loss}
\min_{\theta,~\xi_{\mathbf x,\mathbf z}} ~&\mathop \mathbb E\limits_{\mathbf x\sim P_{data}} L_\theta(\mathbf x) + \lambda \mathop \mathbb E\limits_{\substack{\mathbf x\sim P_{data} \\ \mathbf z\sim P_z}}\xi_{\mathbf x,\mathbf z}\\
{\rm s.t.,}~ &L_\theta(\mathbf x) - \xi_{\mathbf x,\mathbf z} \leq L_\theta(G_{\phi^*}(\mathbf z)) - \Delta(\mathbf x, G_{\phi^*}(\mathbf z))\nonumber\\
&\xi_{\mathbf x,\mathbf z} \ge 0\nonumber
\end{align}
where $\lambda$ is a positive balancing parameter, and $P_{data}(\mathbf x)$ is the data distribution of real samples.  The first term minimizes the expected loss function over data distribution since a smaller loss is preferred on real samples. The second term is the expected error caused by the violation of the constraint. Without loss of generality, we require the loss function should be nonnegative.

%
%

Given a fixed loss function $L_{\theta^*}$, on the other hand, one can solve the following minimization problem to find an optimal generator $G_{\phi^*}$.
\begin{equation}\label{eq:generator}
\min_{\phi}~\mathop \mathbb E\limits_{\substack{\mathbf z\sim P_z(\mathbf z)}} L_{\theta^*}(G_\phi(\mathbf z))
\end{equation}

We can use $P_{G_\phi}$ and $P_{G_{\phi^*}}$ to denote the density of samples generated by $G_\phi(\mathbf z)$ and $G_{\phi^*}(\mathbf z)$ respectively, with $\mathbf z$ being drawn from $P_z(\mathbf z)$. However, for the simplicity of notations, we will use $P_{G}$ and $P_{G^*}$ to denote $P_{G_\phi}$ and $P_{G_{\phi^*}}$ without explicitly mentioning $\phi$ and $\phi^*$ that should be clear in the context.

Finally, let us summarize the above objectives. The LS-GAN optimizes $L_\theta$ and $G_\phi$ alternately by seeking an equilibrium $(\theta^*,\phi^*)$ such that $\theta^*$ minimizes
\begin{align}\label{eq:theta}
&S(\theta,\phi^*)=~\mathop \mathbb E\limits_{\mathbf x\sim P_{data}} L_\theta(\mathbf x) 
+ \lambda \mathop \mathbb E\limits_{\substack{\mathbf x\sim P_{data} \\ \mathbf z_G\sim P_{G^*}}}\big( \Delta(\mathbf x, \mathbf z_G) + L_\theta(\mathbf x) - L_\theta(\mathbf z_G) \big)_+
\end{align}
which is an equivalent form of (\ref{eq:loss}) with $(a)_+=\max(a,0)$, and $\phi^*$ minimizes
\begin{equation}\label{eq:phi}
\begin{aligned}[l]
T(\theta^*,\phi)=\mathop \mathbb E\limits_{\mathbf z_G\sim P_{G}} L_{\theta^*}(\mathbf z_G).
\end{aligned}
\end{equation}
In the next section, we will show the consistency between $P_{G^*}$ and $P_{data}$ for LS-GAN.

\section{Theoretical Analysis: Distributional Consistency}\label{sec:theory}

Suppose $(\theta^*,\phi^*)$ is a Nash equilibrium that jointly solves  (\ref{eq:theta}) and (\ref{eq:phi}).  We will show that as $\lambda\rightarrow +\infty$, the density distribution $P_{G^*}$ of the samples generated by $G_{\phi^*}$ will converge to the real data density $P_{data}$.


First, we have the following definition.

{\noindent {\bf Definition.} \em For any two samples $\mathbf x$ and $\mathbf z$, the loss function $F(\mathbf x)$ is Lipschitz continuous with respect to a distance metric $\Delta$ if
$$
|F(\mathbf x) - F(\mathbf z)|\leq \kappa ~ \Delta(\mathbf x, \mathbf z)
$$
with a bounded Lipschitz constant $\kappa$, i.e, $\kappa<+\infty$.
}


To prove our main result, we assume the following regularity condition on the data density. 
\begin{assumption}\label{asp}
The data density $P_{data}$ is supported in a compact set $\mathcal D$, and it is Lipschitz continuous wrt $\Delta$ with a bounded constant $\kappa<+\infty$.
\end{assumption}

The set of Lipschitz densities with a compact support contain a large family of distributions that are dense in the space of continuous densities.
For example, the density of natural images are defined over a compact set of pixel values, and it can be consider as Lipschitz continuous, since the densities of two similar images are unlikely to change abruptly at an unbounded rate. If real samples are distributed on a manifold (or $P_{data}$ is supported in a manifold), we only require the Lipschitz condition hold on this manifold.  This makes the Lipschitz regularity applicable to the data densities on a thin manifold embedded in the ambient space.

Let us show the existence of Nash equilibrium
such that both the loss function $L_{\theta^*}$ and the density $P_{G^*}$ of generated samples are Lipschitz. Let $\mathcal F_\kappa$ be the class of functions over $\mathcal D$ with a bounded yet sufficiently large Lipschitz constant $\kappa$ such that $P_{data}$ belongs to $\mathcal F_\kappa$.
It is not difficult to show that the space $\mathcal F_\kappa$ is convex and compact
if its member functions are supported in a compact set.
In addition, we note both $S(\theta,\phi)$ and $T(\theta,\phi)$ are convex in $L_\theta$ and in $P_G$.
Then, according to the Sion's theorem \cite{sion1958general}, with $L_{\theta}$ and $P_{G}$ being optimized over $\mathcal F_\kappa$, there exists a Nash equilibrium $(\theta^*,\phi^*)$.  Thus, we have the following lemma.



\begin{lemma}\label{lem2}
Under Assumption~\ref{asp}, there exists a Nash equilibrium $(\theta^*,\phi^*)$ such that both $L_{\theta^*}$ and $P_{G^*}$ are Lipschitz.
\end{lemma}

Now we can prove the main lemma of this paper. The Lipschitz regularity relaxes the strong non-parametric assumption on the GAN's discriminator with infinite capacity to the above weaker Lipschitz assumption for the LS-GAN. This allows us to show the following lemma that establishes the distributional consistency between the optimal $P_{G^*}$ by Problem~(\ref{eq:theta})--(\ref{eq:phi}) and the data density $P_{data}$.




\begin{lemma}\label{lem1}
{Under Assumption~\ref{asp}, for a Nash equilibrium $(\theta^*,\phi^*)$ in Lemma~\ref{lem2}, we have
$$
\int_{\mathbf x}|P_{data}(\mathbf x)-P_{G^*}(\mathbf x)|d\mathbf x \leq \dfrac{2}{\lambda}
$$
Thus, $P_{G^*}(\mathbf x)$ converges to $P_{data}(\mathbf x)$ as $\lambda\rightarrow+\infty$.}
\end{lemma}
The proof of this lemma is given in Appendix~\ref{proof_a}.
\begin{remark}
By letting $\lambda$ go infinitely large, the density $P_{G^*}(\mathbf x)$ of generated samples should exactly match the real data density $P_{data}(\mathbf x)$.  Equivalently, we can simply disregard the first loss minimization term in (\ref{eq:theta}) as it plays no role as $\lambda\rightarrow+\infty$.
\end{remark}


Putting the above two lemmas together, we have the following theorem.
\begin{theorem}\label{thm3}
Under Assumption~\ref{asp}, a Nash equilibrium $(\theta^*,\phi^*)$ exists such that\\
(i) $L_{\theta^*}$ and $P_{G^*}$ are Lipschitz.\\
(ii) $\int_{\mathbf x}|P_{data}(\mathbf x)-P_{G^*}(\mathbf x)|d\mathbf x \leq \dfrac{2}{\lambda}\rightarrow 0$, as $\lambda \rightarrow +\infty$.\\
\end{theorem}

\section{Learning and Generalizability}\label{sec:alg}
The minimization problems (\ref{eq:theta}) and (\ref{eq:phi}) cannot be solved directly since the expectations over the distributions of true data $P_{data}$ and noises $P_z$ are unavailable or intractable.  Instead, one can approximate them with empirical means on a set of finite real examples $\mathcal X_m=\{\mathbf x_1,\cdots,\mathbf x_m\}$ and noise vectors $\mathcal Z_m=\{\mathbf z_1,\cdots,\mathbf z_m\}$ drawn from $P_{data}(\mathbf x)$ and $P_z(\mathbf z)$ respectively.

This results in the following two alternative problems.
\begin{align}\label{eq:theta1}
&\min_{\theta} S_{m} (\theta,\phi^*) \triangleq ~\dfrac{1}{m} \sum_{i=1}^m {L_\theta(\mathbf x_i)}
+ \dfrac{\lambda}{m} \sum_{i=1}^{m} \big( \Delta(\mathbf x_i,G_{\phi^*}(\mathbf z_i)) + L_\theta(\mathbf x_i) - L_\theta(G_{\phi^*}(\mathbf z_i)) \big)_+
\end{align}
and
\begin{equation}\label{eq:phi1}
\begin{aligned}
\min_{\phi} T_k(\theta^*,\phi)=\dfrac{1}{k} \sum_{i=1}^k L_{\theta^*}(G_\phi(\mathbf z_i'))
\end{aligned}
\end{equation}
where the random vectors $\mathcal Z'_k=\{\mathbf z_i'|i=1,\cdots,k\}$ used in (\ref{eq:phi1}) can be different from $\mathcal Z_m$ used in (\ref{eq:theta1}).

The sample mean in the second term of Eq.~(\ref{eq:theta1}) is computed over pairs $(\mathbf x_i, G_{\phi^*}(\mathbf z_i))$ randomly drawn from real and generated samples, which is an approximation to the second expectation term in Eq.~(\ref{eq:theta}).

\subsection{Generalizability}\label{sec:gen}

We have proved the density of generated samples by the LS-GAN is consistent with the real data density in Theorem~\ref{thm3}. This consistency is established based on the two oracle objectives (\ref{eq:theta}) and (\ref{eq:phi}).  However, in practice, the population expectations in these two objectives cannot be computed directly over $P_{data}$ and $P_G$. Instead, they are approximated in (\ref{eq:theta1}) and (\ref{eq:phi1}) by sample means on a finite set of real and generated examples.

This raises the question about the generalizability of the LS-GAN model.  We wonder, with more training examples, if the empirical model trained with finitely many examples can generalize to the oracle model. In particular, we wish to estimate the sample complexity of how many examples are required to sufficiently bound the generalization difference between the empirical and oracle objectives.


Arora et al. \cite{arora2017generalization} has proposed a neural network distance to analyze the generalization ability for the GAN. However, this neural network distance cannot be directly applied here, as it is not related with the objectives that are used to train the LS-GAN. So the generalization ability in terms of the neural network distance does not imply the LS-GAN could also generalize. Thus, a direct generalization analysis of the LS-GAN is required based on its own objectives.

First, let us consider the generalization in terms of $S(\theta,\phi^*)$.
This objective is used to train the loss function $L_\theta$ to distinguish between real and generated samples.
Consider the oracle objective  (\ref{eq:theta}) with the population expectations
$$
S=\min_\theta S(\theta,\phi^*)
$$
and the empirical objective (\ref{eq:theta1}) with the sample means
$$
S_m=\min_\theta S_m(\theta,\phi^*).
$$

We need to show if and how fast the difference $|S_m-S|$ would eventually vanish as the number $m$ of training examples grows.

To this end, we need to define the following notations about the model complexity.
\begin{assumption}\label{asp:loss}
We assume that for LS-GAN,
\begin{itemize}
\item[I.] the loss function $L_\theta(\mathbf x)$ is $\kappa_L$-Lipschitz in its parameter $\theta$, i.e., $|L_\theta(\mathbf x)-L_{\theta'}(\mathbf x)|\leq \kappa_L\|\theta-\theta'\|$ for any $\mathbf x$;
\item[II.] $L_\theta(\mathbf x)$ is $\kappa$-Lipschitz in $\mathbf x$, i.e., $|L_\theta(\mathbf x)-L_\theta(\mathbf x')|\leq\kappa\|\mathbf x-\mathbf x'\|$ for any $\theta$;
\item[III.] the distance between two samples is bounded, i.e., $|\Delta(\mathbf x, \mathbf x')|\leq B_\Delta$.
\end{itemize}
\end{assumption}

Then we can prove the following generalization theorem in a Probably Approximately Correct (PAC) style.
\begin{theorem}\label{thm:generalization}
Under Assumption~\ref{asp:loss}, with at least probability $1-\eta$, we have
$$
|S_m - S|\leq \varepsilon
$$
when the number of samples
$$
m\geq\dfrac{C B_\Delta^2(\kappa+1)^2}{\varepsilon^2} \big(N \log\dfrac{\kappa_L N}{\varepsilon}+\log\dfrac{1}{\eta}\big),$$
where $C$ is a sufficiently large constant, and $N$ is the number of parameters of the loss function such that $\theta\in\mathbb R^N$.
\end{theorem}
The proof of this theorem is given in Appendix~\ref{sec:gen_proof}. This theorem shows the sample complexity to bound the difference between $S$ and $S_m$ is polynomial in the model size $N$, as well as both Lipschitz constants $\log\kappa_L$ and $\kappa$.


Similarly, we can establish the generalizability to train the generator function by considering the empirical objective
$$
T_k = \min_{\phi} T_k(\theta^*,\phi)
$$
and the oracle objective
$$
T = \min_\phi T(\theta^*,\phi)
$$
over empirical and real distributions, respectively.

We use the following notions to characterize the complexity of the generator.
\begin{assumption}\label{asp:generator}
We assume that
\begin{itemize}
\item[I.] The generator function $G_\phi(\mathbf x)$ is $\rho_G$-Lipschitz in its parameter $\phi$, i.e., $|G_\phi(\mathbf z)-G_{\phi'}(\mathbf z)|\leq \rho_G\|\phi-\phi'\|$ for any $\mathbf z$;
\item[II.] Also, we have $G_\phi(\mathbf z)$ is $\rho$-Lipschitz in $\mathbf z$, i.e., $|G_\phi(\mathbf z)-G_\phi(\mathbf z')|\leq\rho\|\mathbf z-\mathbf z'\|$;
\item[III.] The samples $\mathbf z$'s drawn from $P_z$ are bounded, i.e., $\|\mathbf z\|\leq B_z$.
\end{itemize}
\end{assumption}

Then we can prove the following theorem to establish the generalizability of the generator in terms of $T(\theta,\phi)$.
\begin{theorem}
Under Assumption~\ref{asp:generator}, with at least probability $1-\eta$, we have
$$
|T_k - T|\leq \varepsilon
$$
when the number of samples
$$
k\geq\dfrac{C' B_z^2\kappa^2\rho^2}{\varepsilon^2} \big(M\log\dfrac{\kappa_L \rho_G M}{\varepsilon}+\log\dfrac{1}{\eta}\big),
$$
where $C'$ is a sufficiently large constant, and $M$ is the number of parameters of the generator function such that $\phi\in\mathbb R^M$.
\end{theorem}

\subsection{Bounded Lipschitz Constants for Regularization}

Our generalization theory  in Theorem~\ref{thm:generalization} conjectures that the required number of training examples is lower bounded by a polynomial of Lipschitz constants $\kappa_L$ and $\kappa$ of the loss function wrt $\theta$ and $\mathbf x$. This suggests us to bound both constants to reduce the sample complexity of the LS-GAN to improve its generalization performance.


Specifically, bounding the Lipschitz constants $\kappa$ and $\kappa_L$ can be implemented by adding two gradient penalties (I) $\frac{1}{2}\mathbb E_{\mathbf x\sim P_{data}} \|\nabla_\mathbf x L_\theta(\mathbf x)\|^2$ and (II) $\frac{1}{2}\mathbb E_{\mathbf x\sim P_{data}} \|\nabla_{\theta} L_\theta(\mathbf x)\|^2$
to the objective (\ref{eq:theta}) as the surrogate of the Lipschitz constants. For simplicity, we  ignore the second gradient penalty (II) for $\kappa_L$ in experiments, as the sample complexity is only log-linear in it, whose impact on generalization performance is negligible compared with that of $\kappa$. Otherwise, penalizing (II) needs to compute its gradient wrt $\theta$, which is $\mathbb E_{\mathbf x\sim P_{data}} \nabla^2_{\theta} L_\theta(\mathbf x) \nabla_{\theta} L_\theta(\mathbf x)$ with a Hessian matrix $\nabla^2_{\theta}$, and this is usually computationally demanding.

Note that the above gradient penalty differs from that used in \cite{gulrajani2017improved} that aims to constrain the Lipschitz constant $\kappa$ close to one as in the definition of the Wasserstein distance \cite{wgan17}. However, we are motivated to have lower sample complexity by directly minimizing the Lipschitz constant rather than constraining it to one.  Two gradient penalty approaches are thus derived from different theoretical perspectives, and also make practical differences in experiments.


\section{Wasserstein GAN and Generalized LS-GAN}
In this section, we discuss two issues about LS-GAN.  First, we discuss its connection with the Wasserstein GAN (WGAN), and then show that the WGAN is a special case of a generalized form of LS-GAN.


\subsection{Comparison with Wasserstein GAN}\label{sec:wgan}
We notice that the recently proposed Wasserstein GAN (WGAN) \cite{wgan17} uses the Earth-Mover (EM) distance to address the vanishing gradient and saturated JS distance problems in the classic GAN by showing the EM distance is continuous and differentiable almost everywhere.  While both the LS-GAN and the WGAN address these problems from different perspectives that are independently developed almost simultaneously, both turn out to use the Lipschitz regularity in training their GAN models. This constraint plays vital but different roles in the two models.  In the LS-GAN, the Lipschitz regularity naturally arises from the Lipschitz assumption on the data density and the generalization bound. Under this regularity condition, we have proved in Theorem~\ref{thm3} that the density of generated samples matches the underlying data density. On the contrary, the WGAN introduces the Lipschitz constraint from the Kantorovich-Rubinstein duality of the EM distance but it is {\em not} proved in \cite{wgan17} if the density of samples generated by WGAN is consistent with that of real data.

Here we assert that the WGAN also models an underlying Lipschitz density.  To prove this, we restate the WGAN as follows. The WGAN seeks to find a critic $f_w^*$ and a generator $g_\phi^*$ such that
\begin{equation}\label{eq:critic}
\begin{aligned}
f_w^* &= \arg\max_{f_w\in\mathcal F_1} U(f_w,g_\phi^*)
\triangleq\mathbb E_{\mathbf x\sim P_{data}} [f_w(\mathbf x)] - \mathbb E_{\mathbf z\sim P_{\mathbf z}(\mathbf z)} [f_w(g_\phi^*(\mathbf z))]
\end{aligned}
\end{equation}
and
\begin{equation}\label{eq:gen}
g_\phi^*=\arg\max V(f_w^*,g_\phi)\triangleq\mathbb E_{\mathbf z\sim P_{\mathbf z}(\mathbf z)} [f_w^*(g_\phi(\mathbf z))]
\end{equation}

Let $P_{g_\phi^*}$ be the density of samples generated by $g_\phi^*$. Then, we prove the following lemma about the WGAN in Appendix~\ref{appendixC}.

\begin{lemma}\label{lem:wgan}
{Under Assumption~\ref{asp}, given an optimal solution $(f_w^*,g_\phi^*)$ to the WGAN such that $P_{g_\phi^*}$ is Lipschitz, we have
$$
\int_{\mathbf x}|P_{data}(\mathbf x)-P_{g_\phi^*}(\mathbf x)|d\mathbf x=0
$$}
\end{lemma}

This lemma shows both the LS-GAN and the WGAN are based on the same Lipschitz regularity condition.

Although both methods are derived from very different perspectives, it is interesting to make a comparison between their respective forms. Formally, the WGAN seeks to maximize the difference between the first-order moments of $f_w$ under the densities of real and generated examples. In this sense, the WGAN can be considered as a kind of {\em first-order moment} method. Numerically, as shown in the second term of Eq.~(\ref{eq:critic}), $f_w$ tends to be minimized to be arbitrarily small over generated samples, which could make $U(f_w,g_\phi^*)$ be unbounded above. This is why the WGAN must be trained by clipping the network weights of $f_w$ on a bounded box to prevent $U(f_w,g_\phi^*)$ from becoming unbounded above.

On the contrary, the LS-GAN treats real and generated examples in pairs, and maximizes the difference of their losses up to a data-dependant margin. Specifically, as shown in the second term of Eq.~(\ref{eq:theta}), when the loss of a generated sample $\mathbf z_G$ becomes too large wrt that of a paired real example $\mathbf x$, the maximization of $L_\theta(\mathbf z_G)$ will stop if the difference $L_\theta(\mathbf z_G)-L_\theta(\mathbf x)$ exceeds $\Delta(\mathbf x, \mathbf z_G)$. This prevents the minimization problem (\ref{eq:theta}) unbounded below, making it better posed to solve.

More importantly, paring real and generated samples in $(\cdot)_+$ prevents their losses from being decomposed into two separate first-order moments like in the WGAN. The LS-GAN makes pairwise comparison between the losses of real and generated samples, thereby enforcing real and generated samples to coordinate with each other to learn the optimal loss function. Specifically, when a generated sample becomes close to a paired real example, the LS-GAN will stop increasing the difference $L_\theta(\mathbf z_G)-L_\theta(\mathbf x)$ between their losses.


Below we discuss a Generalized LS-GAN (GLS-GAN) model in Section~\ref{sec:glsgan}, and show that both WGAN and LS-GAN are simply two special cases of this GLS-GAN.


\subsection{GLS-GAN: Generalized LS-GAN}\label{sec:glsgan}

In proving Lemma~\ref{lem1}, it is noted that we only have used two properties of $(a)_+$ in the objective function $S_\theta(\theta,\phi^*)$ training the loss function $L_\theta$: 1) $(a)_+ \geq a$ for any $a$; 2) $(a)_+=a$ for $a\geq 0$. This inspires us to generalize the LS-GAN with any alternative cost function $C(a)$ satisfying these two properties, and this will yield the Generalized LS-GAN (GLS-GAN).

We will show that {\bf both LS-GAN and WGAN can be seen as two extreme cases of this GLS-GAN } with two properly defined cost functions.

Formally, if a cost function $C(a)$ satisfies
\begin{itemize}
\item[(I)] $C(a)\geq a$ for any $a\in\mathbb R$ and
\item[(II)] $C(a) = a$ for any $a\in\mathbb R_+$,
\end{itemize}
given a fixed generator $G_{\phi^*}$, we use the following objective
\begin{align}\label{eq:glsgan}
S_C(\theta,\phi^*)=
 \mathop \mathbb E\limits_{\substack{\mathbf x\sim P_{data}(\mathbf x) \\ \mathbf z\sim P_{z}(\mathbf z)}}C\big( \Delta(\mathbf x, G_{\phi^*}(\mathbf z)) + L_\theta(\mathbf x) - L_\theta(G_{\phi^*}(\mathbf z)) \big)\nonumber
\end{align}
to learn $L_\theta(\mathbf x)$, with $S_C$ highlighting its dependency on a chosen cost function $C$.

For simplicity, we only involve the second term in (\ref{eq:theta}) to define the generalized objective $S_C$. But it does not affect the conclusion as the role of the first term in (\ref{eq:theta}) would vanish with $\lambda$ being set to $+\infty$. Following the proof of Lemma~\ref{lem1}, we can prove the following lemma.
\begin{lemma}\label{lem5}
Under Assumption~\ref{asp}, given a Nash equilibrium $(\theta^*,\phi^*)$ jointly minimizing $S_C(\theta,\phi^*)$ and $T(\theta^*,\phi)$ with a cost function $C$ satisfying the above conditions (I) and (II), we have
$$
\int_{\mathbf x}|P_{data}(\mathbf x)-P_{G^*}(\mathbf x)|d\mathbf x = 0.
$$
\end{lemma}

In particular, we can choose a leaky rectified linear function for this cost function, i.e., $C_\nu(a)=\max(a,\nu a)$ with a slope $\nu$. As long as $\nu\in(-\infty,1]$, it is easy to verify $C_\nu(a)$ satisfies these two conditions.

Now the LS-GAN is a special case of this Generalized LS-GAN (GLS-GAN) when $\nu=0$, as $C_0(a)=(a)_+$. We denote this equivalence as

\vspace{1mm}
\centerline{\bf LS-GAN = GLS-GAN($C_0$)}
\vspace{1mm}

What is more interesting is the WGAN, an independently developed GAN model with stable training performance, also becomes a special case of this GLS-GAN with $\nu=1$.  Indeed, when $\nu=1$, $C_1(a)=a$, and
\[
\begin{aligned}
&S_{C_1}(\theta,\phi^*)
= \mathop \mathbb E\limits_{\substack{\mathbf x\sim P_{data}(\mathbf x) \\ \mathbf z\sim P_{z}(\mathbf z)}}\big( \Delta(\mathbf x, G_{\phi^*}(\mathbf z)) + L_\theta(\mathbf x) - L_\theta(G_{\phi^*}(\mathbf z))\big)\\
&=  \mathop \mathbb E\limits_{\mathbf x\sim P_{data}(\mathbf x)} L_\theta(\mathbf x) - \mathop \mathbb E\limits_{\mathbf z\sim P_{z}(\mathbf z)} L_\theta(G_{\phi^*}(\mathbf z))
+\mathop \mathbb E\limits_{\substack{\mathbf x\sim P_{data}(\mathbf x) \\ \mathbf z\sim P_{z}(\mathbf z)}}\Delta(\mathbf x, G_{\phi^*}(\mathbf z))
\end{aligned}
\]
Since the last term $\mathop \mathbb E_{{\mathbf x\sim P_{data}, \mathbf z\sim P_{z}}}\Delta(\mathbf x, G_{\phi^*}(\mathbf z))$ is a const, irrespective of $L_\theta$, it can be discarded without affecting optimization over $L_\theta$. Thus, we have
$$
S_{C_1}(\theta,\phi^*)=\mathop \mathbb E\limits_{\mathbf x\sim P_{data}(\mathbf x)} L_\theta(\mathbf x) - \mathop \mathbb E\limits_{\mathbf z\sim P_{z}(\mathbf z)} L_\theta(G_{\phi^*}(\mathbf z))
$$

By comparing this $S_{C_1}$ with $U$ in (\ref{eq:critic}), it is not hard to see that the WGAN is equivalent to the GLS-GAN with $C_1$, with the critic function $f_w$ being equivalent to $-L_\theta$ \footnote{the minus sign exists as the $U$ is maximized over $f_w$ in the WGAN. On the contrary, in the GLS-GAN, $S_C$ is minimized over $L_\theta$.}. Thus we have

\vspace{1mm}\centerline{\bf WGAN = GLS-GAN($C_1$)}\vspace{1mm}

Therefore, by varying the slope $\nu$ in $(-\infty,1]$, we will obtain a family of the GLS-GANs with varied $C_\nu$ beyond the LS-GAN and the WGAN.
Of course, it is unnecessary to limit $C(a)$ to a leaky rectified linear function. We can explore more cost functions as long as they satisfy the two conditions (I) and (II).

In experiments, we will demonstrate the GLS-GAN has competitive generalization performance on generating new images (c.f. Section~\ref{sec:eval_gen}).

\section{Non-Parametric Analysis}\label{sec:nonparam}
Now we can characterize the optimal loss functions learned from the objective (\ref{eq:theta1}), and this will provide us an insight into the LS-GAN model.

We generalize the non-parametric maximum likelihood method in \cite{carando2009nonparametric} and consider non-parametric solutions to the optimal loss function by minimizing (\ref{eq:theta1}) over the whole class of Lipschitz loss functions.


Let $\mathbf x^{(1)}=\mathbf x_1, \mathbf x^{(2)}=\mathbf x_2, \cdots,\mathbf x^{(m)}=\mathbf x_m, \mathbf x^{(m+1)}=G_{\phi^*}(\mathbf z_1),\cdots,\mathbf x^{(2m)}=G_{\phi^*}(\mathbf z_m)$, i.e., the first $n$ data points are real examples and the rest $m$ are generated samples.  Then we have the following theorem.
\begin{theorem}\label{thm_nonparam}
The following functions $\widehat L_{\theta^*}$ and $\widetilde L_{\theta^*}$ both minimize $S_{m}(\theta,\phi^*)$ in $\mathcal F_\kappa$:
\begin{equation}\label{eq:param}
\begin{aligned}
&\widehat L_{\theta^*}(\mathbf x) = \max_{1\leq i\leq 2m}\big\{\big(l_i^*-\kappa\Delta(\mathbf x,\mathbf x^{(i)})\big)_+\big\}, \\
&\widetilde L_{\theta^*}(\mathbf x) = \min_{1\leq i\leq 2m}\big\{l_i^*+\kappa\Delta(\mathbf x,\mathbf x^{(i)})\}
\end{aligned}
\end{equation}
with the parameters $\theta^*=[l_1^*,\cdots,l_{2m}^*]\in\mathbb R^{2m}$. They are supported in the convex hull of $\{\mathbf x^{(1)}, \cdots, \mathbf x^{(2m)}\}$, and we have
$$\widehat L_{\theta^*}(\mathbf x^{(i)})=\widetilde L_{\theta^*}(\mathbf x^{(i)})=l_i^*$$
for $i=1,\cdots,2m$, i.e., their values coincide on $\{\mathbf x^{(1)},\mathbf x^{(2)},\cdots,\mathbf x^{(2m)}\}$.
\end{theorem}
The proof of this theorem is given in the appendix.

From the theorem, it is not hard to show that any convex combination of these two forms attains the same value of $S_{m}$, and is also a global minimizer.  Thus, we have the following corollary.

\begin{corollary}
All the functions in
$$\mathcal L_{\theta^*}=\{\gamma \widehat L_{\theta^*} + (1-\gamma) \widetilde L_{\theta^*}|0\leq \gamma \leq 1\}\subset\mathcal F_\kappa$$
 minimize $S_{m}$ in $\mathcal F_\kappa$.
\end{corollary}

\begin{figure}[t!]
    \centering
        \includegraphics[width=0.5\linewidth]{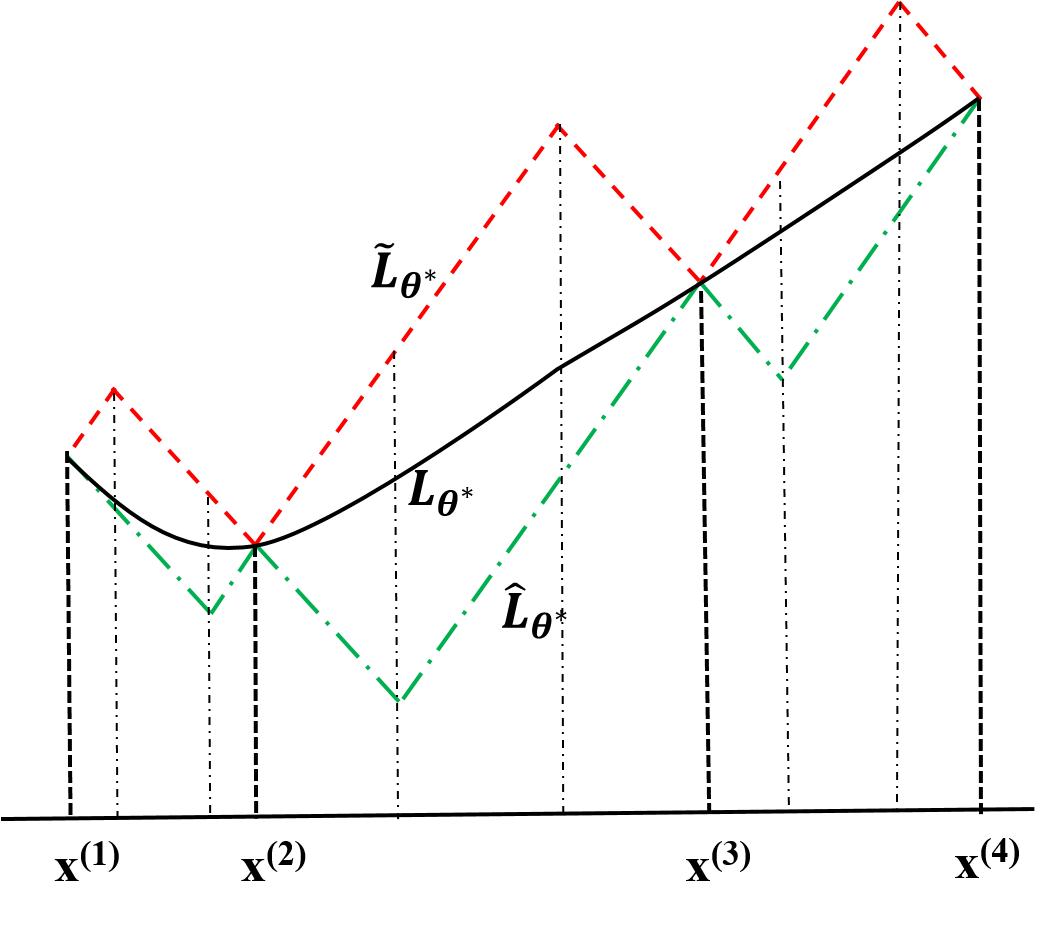}
        \caption{Comparison between two optimal loss functions $\widetilde L_{\theta^*}$ and $\widehat L_{\theta^*}$ in $\mathcal F_\kappa$ for LS-GAN. They are upper and lower bounds of the class of optimal loss functions $L_{\theta^*}$ to Problem (\ref{eq:theta1}). Both the upper and the lower bounds are cone-shaped, and have non-vanishing gradient almost everywhere. Specifically,
        in this one-dimensional example, both bounds are piecewise linear, having a slope of $\pm\kappa$ almost everywhere.}\label{fig:fun_L}
\end{figure}

This shows that the global minimizer is not unique. Moreover, through the proof of Theorem \ref{thm_nonparam}, one can find that $\widetilde L_{\theta^*}(\mathbf x)$ and $\widehat L_{\theta^*}(\mathbf x)$ are the upper and lower bound of any optimal loss function solution to the problem (\ref{eq:theta1}). In particular, we have the following corollary.
\begin{corollary}\label{cor2}
For any $L_{\theta^*}(\mathbf x)\in\mathcal F_\kappa$ that minimizes $S_{m}$, the corresponding $\widehat L_{\theta^*}(\mathbf x)$ and $\widetilde L_{\theta^*}(\mathbf x)$ are the lower and upper bounds of $L_{\theta^*}(\mathbf x)$, i.e.,
$$
\widehat L_{\theta^*}(\mathbf x) \leq L_{\theta^*}(\mathbf x) \leq \widetilde L_{\theta^*}(\mathbf x)
$$
\end{corollary}
The proof is given in Appendix~\ref{proof_b}.

The parameters $\theta^*=[l_1^*,\cdots,l_{2m}^*]$ in (\ref{eq:param}) can be sought by minimizing
\begin{align}\label{eq:theta_param}
& S_{m} (\phi^*,\theta) \triangleq ~\dfrac{1}{m} \sum_{i=1}^m {l_i} + \dfrac{\lambda}{m} \sum_{i=1}^{m} \big( \Delta_{i,m+i}
+ l_i - l_{m+i} \big)_+\nonumber \\
&~{\rm s.t.,}~|l_i - l_{i'}| \leq \kappa\Delta(\mathbf x^{(i)},\mathbf x^{(i')})\nonumber\\
&~~~~~~~~~l_i\geq 0,~i,i'=1,\cdots,2m
\end{align}
where $\Delta_{i,j}$ is short for $\Delta(\mathbf x^{(i)},\mathbf x^{(j)})$, and the constraints are imposed to ensure the learned loss functions stay in $\mathcal F_\kappa$.
With a greater value of $\kappa$, a larger class of loss function will be sought.  Thus, one can control the modeling ability of the loss function by setting a proper value to $\kappa$.

Problem (\ref{eq:theta_param}) is a typical linear programming problem. In principle, one can solve this problem to obtain a non-parametric loss function for the LS-GAN.
Unfortunately, it consists of a large number of constraints, whose scale is at an order of $\left( {\begin{array}{*{10}{c}}
   {2m}  \\
   2  \\
\end{array}} \right)$. This prevents us from using (\ref{eq:theta_param}) directly to solve an optimal non-parametric LS-GAN model with a very large number of training examples.
On the contrary, a more tractable solution is to use a parameterized network to solve the optimization problem (\ref{eq:theta1}) constrained in $\mathcal L_\kappa$, and iteratively update parameterized $L_\theta$ and $G_\phi$ with the gradient descent method.

Although the non-parametric solution cannot be solved directly, it is valuable in shedding some light on what kind of the loss function would be learned by a deep network. It is well known that the training of the classic GAN generator suffers from vanishing gradient problem as the discriminator can be optimized very quickly.  Recent study \cite{wgan17} has revealed that this is caused by using the Jensen-Shannon (JS) distance that becomes locally saturated and gets vanishing gradient to train the GAN generator if the discriminator is over-trained. Similar problem has also been found in the energy-based GAN (EBGAN) \cite{zhao2016energy} as it minimizes the total variation that is not continuous or (sub-)differentiable if the corresponding discriminator is fully optimized \cite{wgan17}.

On the contrary, as revealed in Theorem~\ref{thm_nonparam} and illustrated in Figure~\ref{fig:fun_L}, both the upper and lower bounds of the optimal loss function of the LS-GAN are cone-shaped (in terms of $\Delta(\mathbf x, \mathbf x^{(i)})$ that defines the Lipschitz continuity), and have non-vanishing gradient almost everywhere. Moreover, Problem (\ref{eq:theta_param}) only contains {\em linear objective and constraints}; this is contrary to the classic GAN that involves logistic loss terms that are prone to saturation with vanishing gradient. Thus, an optimal loss function that is properly sought in $\mathcal L_\kappa$ as shown in Figure~\ref{fig:fun_L} is unlikely to saturate between these two bounds, and it should be able to provide sufficient gradient to update the generator by descending (\ref{eq:phi1}) even if it has been trained till optimality. Our experiment also shows that, even if the loss function is quickly trained to optimality, it can still provide sufficient gradient to continuously update the generator in the LS-GAN (see Figure~\ref{fig:gradG}).

\section{Conditional LS-GAN}\label{sec:clsgan}
The LS-GAN can easily be generalized to produce a sample based on a given condition $\mathbf y$, yielding a new paradigm of Conditional LS-GAN (CLS-GAN).

For example, if the condition is an image class, the CLS-GAN seeks to produce images of the given class; otherwise, if a text description is given as a condition, the model attempts to generate images aligned with the given description.  This gives us more flexibility in controlling what samples to be generated.

Formally, the generator of CLS-GAN takes a condition vector $\mathbf y$ as input along with a noise vector $\mathbf z$ to produce a sample $G_\phi(\mathbf z, \mathbf y)$.
To train the model, we define a loss function $L_\theta(\mathbf x, \mathbf y)$ to measure the degree of the misalignment between a data sample $\mathbf x$ and a given condition $\mathbf y$.

For a real example $\mathbf x$ aligned with the condition $\mathbf y$, its loss function should be smaller than that of a generated sample by a margin of $\Delta(\mathbf x, G_\phi(\mathbf z,\mathbf y))$.  This results in the following constraint,
\begin{equation}
L_\theta(\mathbf x, \mathbf y) \leq L_\theta(G_\phi(\mathbf z,\mathbf y),\mathbf y) - \Delta(\mathbf x, G_\phi(\mathbf z,\mathbf y))
\end{equation}



Like the LS-GAN, this type of constraint yields the following non-zero-sum game to train the CLS-GAN, which seeks a Nash equilibrium $(\theta^*,\phi^*)$ so that
$\theta^*$ minimizes
\begin{align}\label{eq:clsgantheta}
S(\theta,\phi^*)&= ~\mathop \mathbb E\limits_{(\mathbf x,\mathbf y)\sim P_{data}} L_\theta(\mathbf x, \mathbf y) \\\nonumber
&+ \lambda \mathop \mathbb E\limits_{\substack{(\mathbf x,\mathbf y)\sim P_{data} \\\nonumber
\mathbf z\sim P_z}}\big( \Delta(\mathbf x, G_{\phi^*}(\mathbf z,\mathbf y)) + L_\theta(\mathbf x,\mathbf y) \\\nonumber
&- L_\theta(G_{\phi^*}(\mathbf z,\mathbf y),\mathbf y) \big)_+
\end{align}
and $\phi^*$ minimizes
\begin{align}\label{eq:clsganphi}
T(\theta^*,\phi)&=  \mathop \mathbb E\limits_{\substack{\mathbf y\sim P_{data} \\
\mathbf z\sim P_z}} L_{\theta^*}(G_\phi(\mathbf z,\mathbf y),\mathbf y)
\end{align}
where $P_{data}$ denotes either the joint data distribution over $(\mathbf x, \mathbf y)$ in (\ref{eq:clsgantheta}) or its marginal distribution over $\mathbf y$ in (\ref{eq:clsganphi}).

Playing the above game will lead to a trained pair of loss function $L_{\theta^*}$ and generator $G_{\phi^*}$.
We can show that the learned generator $G_{\phi^*}(\mathbf z, \mathbf y)$ can produce samples whose distribution follows the true data density $P_{data}(\mathbf x|\mathbf y)$ for a given condition $\mathbf y$.

To prove this, we say a loss function $L_\theta(\mathbf x, \mathbf y)$ is Lipschitz if it is Lipschitz continuous in its first argument $\mathbf x$. We also impose the following regularity condition on the conditional density $P_{data}(\mathbf x|\mathbf y)$.
\begin{assumption}\label{asp2}
For each $\mathbf y$, the conditional density $P_{data}(\mathbf x|\mathbf y)$ is Lipschitz, and is supported in a convex compact set of $\mathbf x$.
\end{assumption}

Then it is not difficult to prove the following theorem, which shows that the conditional density $P_{G^*}(\mathbf x|\mathbf y)$ becomes $P_{data}(\mathbf x|\mathbf y)$ as $\lambda\rightarrow +\infty$. Here $P_{G^*}(\mathbf x|\mathbf y)$ denotes the density of samples generated by $G_{\phi^*}(\mathbf z,\mathbf y)$ with sampled random noise $\mathbf z$.

\begin{theorem}\label{thm4}
Under Assumption~\ref{asp2}, a Nash equilibrium $(\theta^*,\phi^*)$ exists such that\\
(i) $L_{\theta^*}(\mathbf x,\mathbf y)$ is Lipschitz continuous in $\mathbf x$ for each $\mathbf y$;\\
(ii) $P_{G^*}(\mathbf x|\mathbf y)$ is Lipschitz continuous;\\
(iii) $\int_{\mathbf x}|P_{data}(\mathbf x|\mathbf y)-P_{G^*}(\mathbf x|\mathbf y)|d\mathbf x \leq \dfrac{2}{\lambda}$.
\end{theorem}

In addition, similar upper and lower bounds can be derived to characterize the learned conditional loss function $L_\theta(\mathbf x, \mathbf y)$ following the same idea for LS-GAN.

A useful byproduct of the CLS-GAN is one can use the learned loss function $L_{\theta^*}(\mathbf x, \mathbf y)$ to predict the label of an example $\mathbf x$ by
\begin{equation}
\mathbf y^* = \arg\min\limits_{\mathbf y} L_{\theta^*}(\mathbf x, \mathbf y)
\end{equation}

The advantage of such a CLS-GAN classifier is it is trained with both labeled and generated examples, the latter of which can improve the training of the classifier by revealing more potential variations within different classes of samples. It also provides a way to evaluate the model based on its classification performance. This is an objective metric we can use to assess the quality of feature representations learned by the model.

For a classification task, a suitable value should be set to $\lambda$.  Although Theorem~\ref{thm4} shows $P_{G^*}$ would converge to the true conditional density $P_{data}$ by increasing $\lambda$, it only ensures it is a good generative rather than classification model.  However, a too large value of $\lambda$ tends to ignore the first loss minimization term of (\ref{eq:clsgantheta}) that plays an important role in minimizing classification error.  Thus, a trade-off should be made to balance between classification and generation objectives.


\subsection{Semi-Supervised LS-GAN}\label{sec:ssl}
The above CLS-GAN can be considered as a fully supervised model to classify examples into different classes.
It can also be extended to a Semi-Supervised model by incorporating unlabeled examples.

Suppose we have $c$ classes indexed by $\{1,2,\cdots,c\}$. In the CLS-GAN, for each class, we choose a loss function that, for example, can be defined as the negative log-softmax,
$$
L_\theta(\mathbf x, \mathbf y=l) = -\log\dfrac{\exp(\mathbf a_l(\mathbf x))}{\sum_{l=1}^c \exp(\mathbf a_l(\mathbf x))}
$$
where $\mathbf a_l(\mathbf x)$ is the $l$th activation output from a network layer.

Suppose we also have unlabeled examples available, and we can define a new loss function for these unlabeled examples so that they can be involved in training the CLS-GAN. Consider an unlabeled example $\mathbf x$, its groundtruth label is unknown.  However, the best guess of its label can be made by choosing the one that minimizes $L_\theta(\mathbf x, \mathbf y=l)$ over $l$, and this inspires us to define the following loss function for the unlabeled example as
$$
L_\theta^{\bf ul}(\mathbf x)\triangleq\min_l L_\theta(\mathbf x, \mathbf y=l)
$$
Here we modify $L_\theta(\mathbf x, \mathbf y=l)$ to $-\log\frac{\exp(\mathbf a_l(\mathbf x))}{1+\sum_{l=1}^c \exp(\mathbf a_l(\mathbf x))}$ so $\frac{1}{1+\sum_{l=1}^c \exp(\mathbf a_l(\mathbf x))}$ can be viewed as the probability that $\mathbf x$ does not belong to any known label.

Then we have the following loss-sensitive objective that explores unlabeled examples to train the CLS-GAN,
\begin{align}
&S^{\bf ul}(\theta,\phi^*) \triangleq
\mathop \mathbb E\limits_{\substack{\mathbf x\sim P_{data}(\mathbf x) \\\nonumber
\mathbf z\sim P_z(\mathbf z)}} \big(\Delta(\mathbf x, G_{\phi^*}(\mathbf z))+L_\theta^{\bf ul}(\mathbf x)-L_\theta^{\bf ul}(G_{\phi^*}(\mathbf z))\big)_+
\end{align}

This objective is combined with $S(\theta,\phi^*)$ defined in (\ref{eq:clsgantheta}) to train the loss function network by minimizing
$$
S(\theta,\phi^*)+\gamma S^{\bf ul}(\theta,\phi^*)
$$
where $\gamma$ is a positive hyperparameter balancing the contributions from labeled and labeled examples.

The idea of extending the GAN for semi-supervised learning has been proposed by Odena \cite{odena2016semi} and
Salimans et al. \cite{salimans2016improved}, where generated samples are assigned to an artificial class, and unlabeled examples are treated as the negative examples.
Our proposed semi-supervised learning differs in creating a new loss function for unlabeled examples from the losses for existing classes, by minimizing which we make the best guess of the classes of unlabeled examples. The guessed labeled will provide additional information to train the CLS-GAN model, and the updated model will in turn improve the guess over the training course.
The experiments in the following section will show that this approach can generate very competitive performance especially when the labeled data is very limited.

\section{Experiments}\label{sec:exp}

Objective evaluation of a data generative model is not an easy task as there is no consensus criteria to quantify the quality of generated samples.  For this reason, we will make a qualitative analysis of generated images, and use image classification to quantitatively evaluate the resultant LS-GAN model.

First, we will assess the quality of generated images by the LS-GAN in comparison with the classic GAN model. Then, we will make an objective evaluation on the CLS-GAN to classify images. This task evaluates the quality of feature representations learned by the CLS-GAN in terms of its classification accuracy directly. 

Finally, we will assess the generalizability of various GAN models in generating new images out of training examples by proposing the Minimum Reconstruction Error (MRE) on a separate test set.

\subsection{Architectures}


\begin{table}[t]
\caption{The Network architecture used in CLS-GAN for training CIFAR-10 and SVHN, where BN stands for batch normalization, LeakyReLU for Leaky Rectifier with a slope of 0.2 for negative value, and
``3c1s96o Conv." means a $3\times 3$ convolution kernel with stride $1$ and $96$ outputs, while "UpConv." denotes the fractionally-stride convolution.}
\label{tab:arch}
\begin{center}
\subtable[Loss Function Network]{
\begin{tabular}{c}    \toprule
Input $32\times 32\times 3$\\\midrule
3c1s96o Conv. BN LeakyReLU\\
3c1s96o Conv. BN LeakyReLU\\
4c2s96o Conv. BN LeakyReLU\\\midrule
3c1s192o Conv. BN LeakyReLU\\
3c1s192o Conv. BN LeakyReLU\\
4c2s192o Conv. BN LeakyReLU\\\midrule
3c1s192o Conv. BN LeakyReLU\\
3c1s192o Conv. BN LeakyReLU\\
1c1s192o Conv. BN LeakyReLU\\\midrule
global meanpool\\\midrule
Output $1\times 1\times 10$\\\bottomrule
\end{tabular}}\vspace{2mm}
\subtable[Generator Network]{
\begin{tabular}{c}    \toprule
Input $100$-D random vector + $10$-D one-hot vector\\\midrule
4c1s512o UpConv. BN LeakyReLU\\
4c2s256o UpConv. BN LeakyReLU\\
4c2s128o UpConv. BN LeakyReLU\\
4c2s3o UpConv. BN LeakyReLU\\\midrule
Elementwise Tanh\\\midrule
Output $32 \times 32 \times 3$\\\bottomrule
\end{tabular}}
\end{center}
\end{table}

We adopted the ideas behind the network architecture for the DCGAN \cite{radford2015unsupervised} to build the generator and the loss function networks.  Compared with the conventional CNNs, maxpooling layers were replaced with strided convolutions in both networks, and fractionally-strided convolutions were used in the generator network to upsample feature maps across layers to finer resolutions.  Batch-normalization layers were added in both networks between convolutional layers, and fully connected layers were removed from these networks.

However, unlike the DCGAN, the LS-GAN model (unconditional version in Section~\ref{sec:lsgan}) did not use a sigmoid layer as the output for the loss function network.  Instead, we removed it and directly output the activation before the removed sigmoid layer.


On the other hand, for the loss function network in CLS-GAN, a global mean-pooling layer was added on top of convolutional layers.  This produced a $1\times 1$ feature map that output the conditional loss $L_\theta(\mathbf x, \mathbf y)$ on different classes $\mathbf y$.
In the generator network, Tanh was used to produce images whose pixel values are scaled to $[-1,1]$. Thus, all image examples in datasets were preprocessed to have their pixel values in $[-1,1]$.
More details about the design of network architectures can be found in literature \cite{radford2015unsupervised}.

Table~\ref{tab:arch} shows the network architecture for the CLS-GAN model on CIFAR-10 and SVHN datasets in the experiments.  In particular, the architecture of the loss function network was adapted from that used in \cite{springenberg2015unsupervised} with nine hidden layers.

\subsection{Training Details}
The models were trained in a mini-batch of $64$ images, and their weights were initialized from a zero-mean Gaussian distribution with a standard deviation of $0.02$. The Adam optimizer \cite{kingma2014adam} was used to train the network with initial learning rate and $\beta_1$ being set to $10^{-3}$ and $0.5$ respectively, while the learning rate was annealed every $25$ epochs by a factor of $0.8$. The other hyperparameters such as $\gamma$ and $\lambda$ were chosen based on an independent validation set held out from training examples.

\begin{figure}[t]
\centering
\subfigure[Inception]{
\begin{minipage}{0.35\linewidth}
\begin{center}
	   \includegraphics[width=0.7\linewidth]{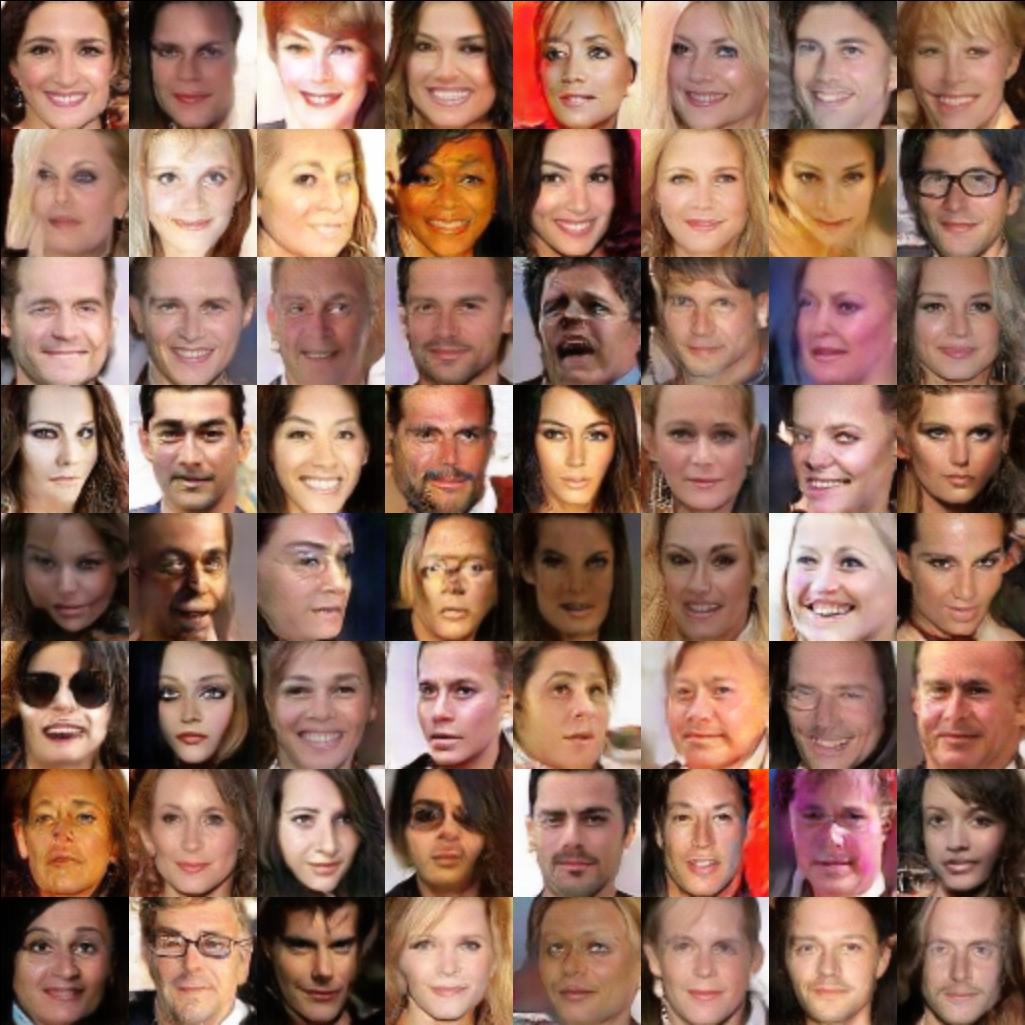}
\end{center}
\end{minipage}}
\subfigure[VGG-16]{
\begin{minipage}{0.35\linewidth}
\begin{center}
	   \includegraphics[width=0.7\linewidth]{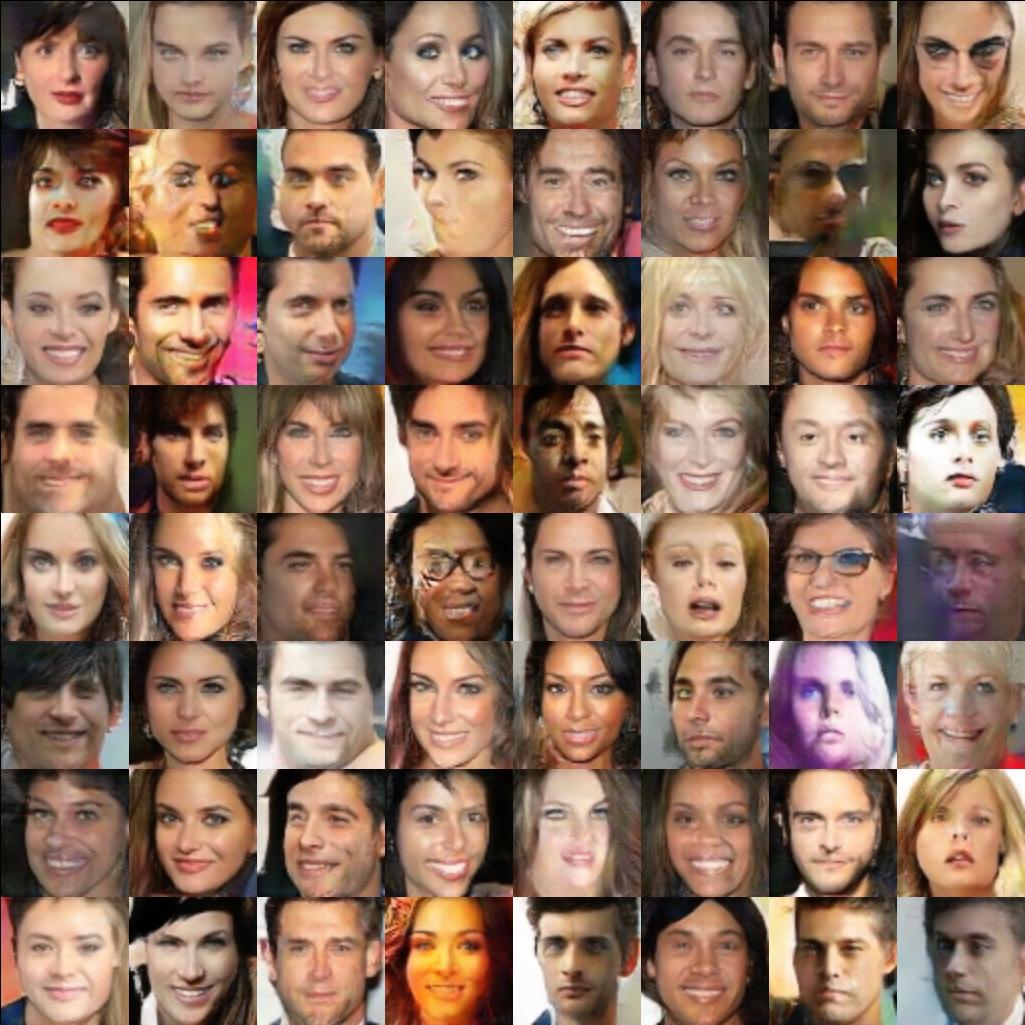}
\end{center}
\end{minipage}}
   \caption{Images generated by the LS-GAN on the CelebA dataset, in which the margin is computed as the distance between the features extracted from the Inception and VGG-16 networks. Images are resized to $128\times 128$ to fit the input size of both networks.}\label{fig:inception_vgg}
\end{figure}

We also tested various forms of loss margins $\Delta(\cdot,\cdot)$ between real and fake samples.
For example, we tried the $L_p$ distance between image representations as the margin, and found the best result can be achieved when $p=1$.
The distance between convolutional features was supposed to capture perceptual dissimilarity between images.
But we should avoid a direct use of the convolutional features from the loss function network, since we found they would tend to collapse to a trivial point as the loss margin vanishes.
The feature maps from a separate pretrained deep network, such as Inception and VGG-16 networks, could be a better choice to define the loss margin. Figure~\ref{fig:inception_vgg} shows the images generated by LS-GAN on CelebA with the inception and VGG-16 margins.

However, for a fair comparison, we did not use these external deep networks in other experiments on image generation and classification tasks. We simply used the distance between raw images as the loss margin, and it still achieved competitive results. This demonstrates the robustness of the proposed method without having to choose a sophisticated loss margin. This is also consistent with our theoretical analysis where we do not assume any particular form of loss margin to prove the results.

%

For the generator network of LS-GAN, it took a $100$-dimensional random vector drawn from Unif$[-1,1]$ as input. For the
CLS-GAN generator, an one-hot vector encoding the image class condition was concatenated with the sampled random vector. The CLS-GAN was trained by involving both unlabeled and labeled examples as in Section~\ref{sec:clsgan}. This was compared against the other state-of-the-art supervised and semi-supervised models.

\subsection{Generated Images by LS-GAN}

First we made a qualitative comparison between the images generated by the DCGAN and the LS-GAN on the celebA dataset.

Figure~\ref{fig:celebA} compares the visual quality of images generated by LS-GAN and DCGAN after they were trained for $25$ epochs, and
there was no perceptible difference between the qualities of their generated images.

\begin{figure}[t!]
\centering
\subfigure[DCGAN]{
\begin{minipage}{0.35\linewidth}
\begin{center}
	   \includegraphics[width=0.7\linewidth]{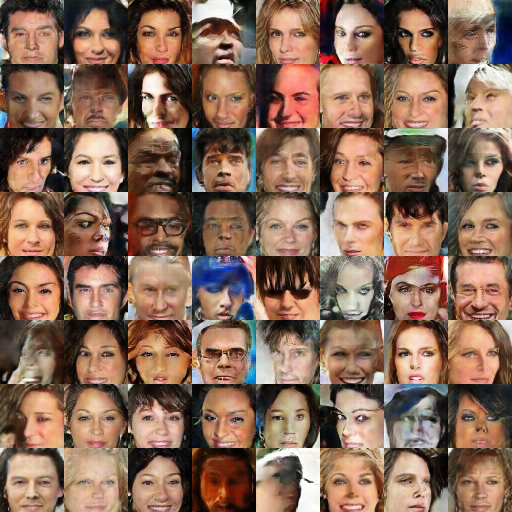}
\end{center}
\end{minipage}}
\subfigure[LS-GAN]{
\begin{minipage}{0.35\linewidth}
\begin{center} \includegraphics[width=0.7\linewidth]{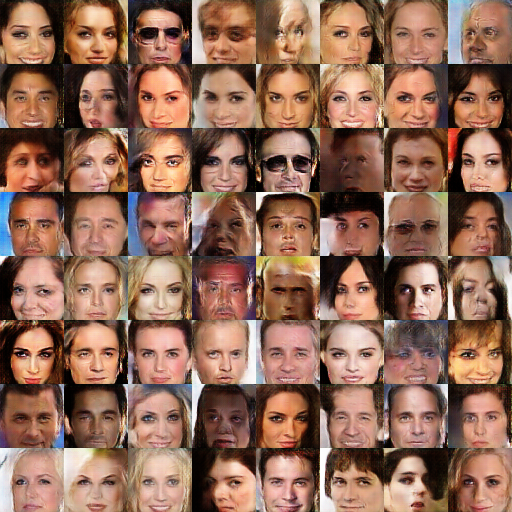}
\end{center}
\end{minipage}}
   \caption{Images generated by the DCGAN and the LS-GAN on the CelebA dataset. The results are obtained after $25$ epochs of training the models.}\label{fig:celebA}
\end{figure}

However, the DCGAN architecture has been exhaustively fine-tuned in terms of the classic GAN training criterion to maximize the image generation performance. It was susceptible that its architecture could be fragile if we make some change to it. Here we tested if the LS-GAN can be more robust than the DCGAN when a structure change was made.

For example, one of the most key components in the DCGAN is the batch normalization inserted between the fractional convolution layers in the generator network. It has been reported in literature \cite{salimans2016improved} that the batch normalization not only plays a key role in training the DCGAN model, but also prevents the mode collapse of the generator into few data points.

The results were illustrated in Figure~\ref{fig:celebA_noBN}. If one removed the batch normalization layers from the generator, the DCGAN would collapse without producing any face images.  On the contrary, the LS-GAN still performed very well even if these batch normalization layers were removed, and there was { no perceived deterioration or mode collapse of the generated images}. This shows that the LS-GAN was more resilient than the DCGAN.

\begin{figure}[t]
\centering
\subfigure[DCGAN]{
\begin{minipage}{0.35\linewidth}
\begin{center}
	   \includegraphics[width=0.7\linewidth]{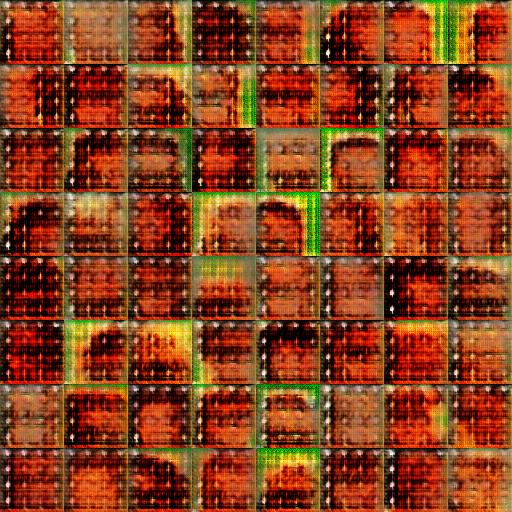}
\end{center}
\end{minipage}}
\subfigure[LS-GAN]{
\begin{minipage}{0.35\linewidth}
\begin{center}
	   \includegraphics[width=0.7\linewidth]{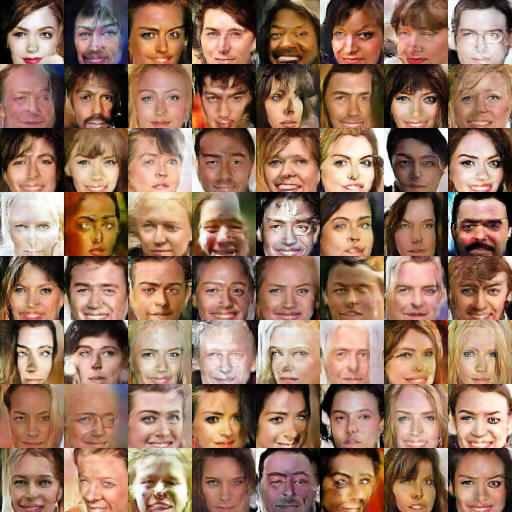}
\end{center}
\end{minipage}}
   \caption{Images generated by the DCGAN and the LS-GAN on the CelebA dataset without batch normalization for the generator networks. The results are obtained after $25$ epochs of training the models.}\label{fig:celebA_noBN}
\end{figure}

We also analyzed the magnitude ($\ell_2$ norm) of the generator's gradient (in logarithmic scale) in Figure~\ref{fig:gradG} over iterations. With the loss function being updated every iteration, the generator was only updated every $1$, $3$, and $5$ iterations.
From the figure, we note that the magnitude of the generator's gradient, no matter how frequently the loss function was updated, gradually increased until it stopped at the same level. This implies the objective function to update the generator tended to be linear rather than saturated through the training process, which was consistent with our non-parametric analysis of the optimal loss function. Thus, it provided sufficient gradient to continuously update the generator. Furthermore, we compared the images generated with different frequencies of updating the loss function in Figure~\ref{fig:freq_gen}, where there was no noticeable difference in the visual quality.  This shows the LS-GAN was not affected by over-trained loss function in experiments.

\begin{figure}[t!]
    \centering
        \includegraphics[width=0.6\linewidth, trim={0cm 8cm 0cm 8cm},clip=true]{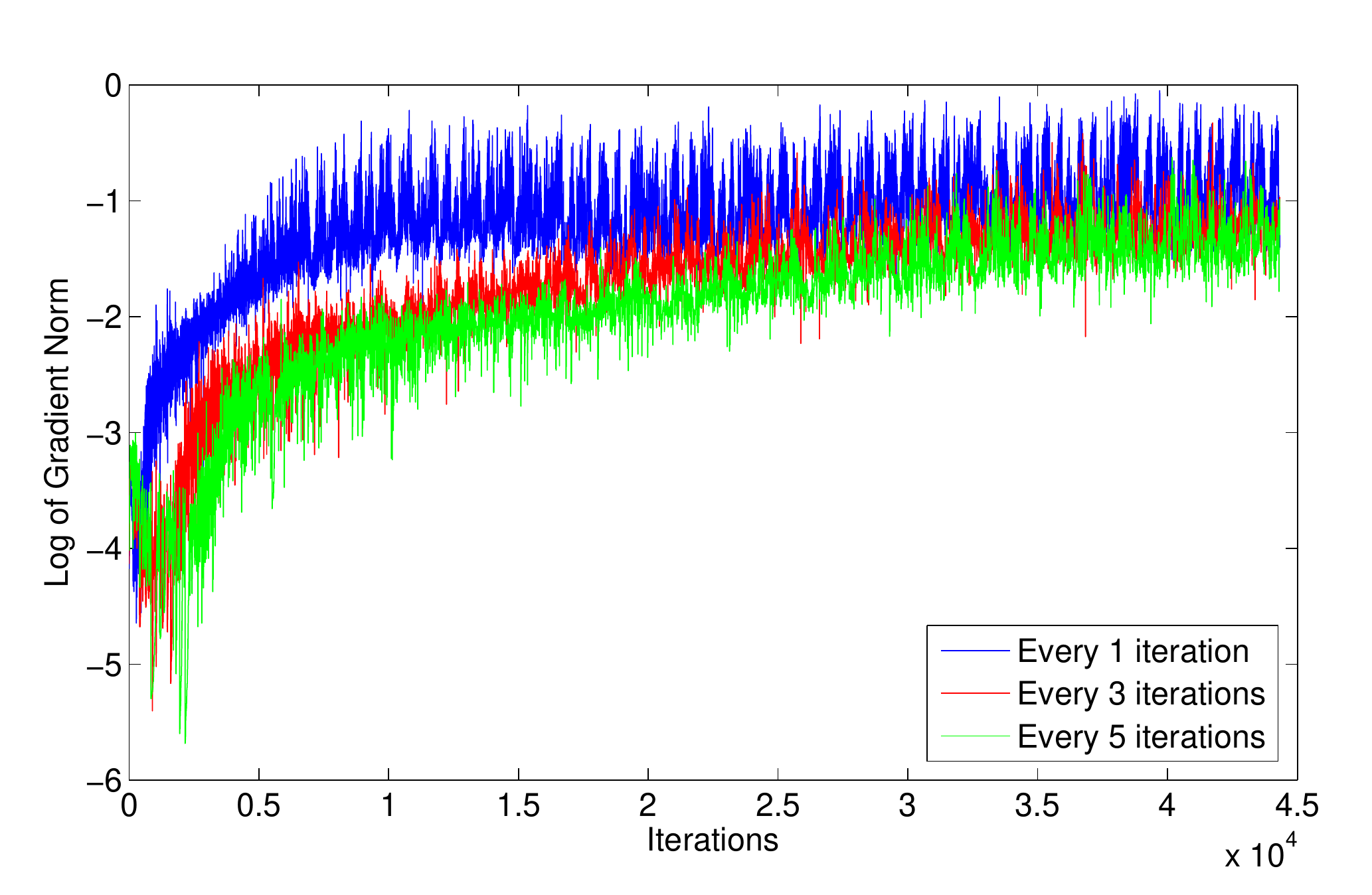}
        \caption{The log of the generator's gradient norm over iterations. The generator is updated every 1, 3, and 5 iterations while the loss function is updated every iteration.  The loss function can be quickly updated to be optimal, and the figure shows the generator's gradient does not vanish even if the loss function is well trained.}\label{fig:gradG}
\end{figure}

\begin{figure}[t]
\centering
\subfigure[]{
\begin{minipage}{0.35\linewidth}
\begin{center}
	   \includegraphics[width=0.7\linewidth]{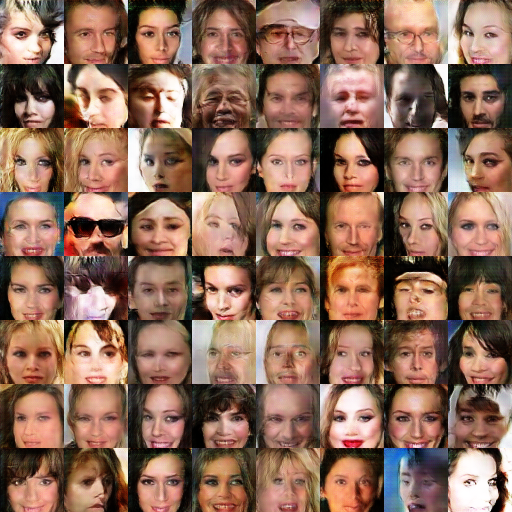}
\end{center}
\end{minipage}}
\subfigure[]{
\begin{minipage}{0.35\linewidth}
\begin{center} \includegraphics[width=0.7\linewidth]{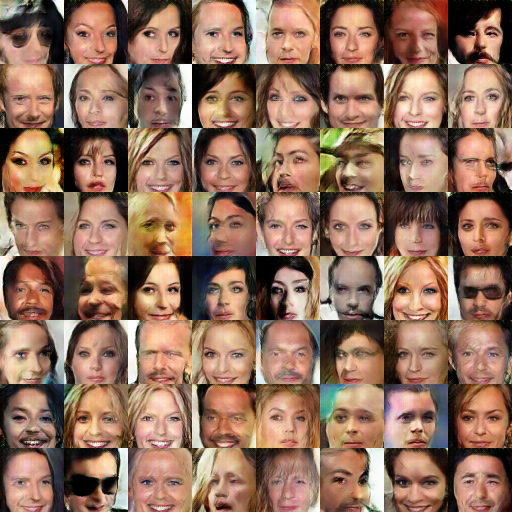}
\end{center}
\end{minipage}}
   \caption{Images generated by the LS-GAN on CelebA, where its generator is updated every three times (a) and every five times (b) the discriminator is updated.}\label{fig:freq_gen}
\end{figure}

\subsection{Image Classification}
We conducted experiments on CIFAR-10 and SVHN to compare the classification accuracy of LS-GAN with the other approaches.


\subsubsection{CIFAR-10}


\begin{table}[]
\caption{Classification accuracies on CIFAR-10 dataset. Accuracies with all training examples labeled (all) and with only 400 labeled examples per class (400) are reported. The best result is highlighted in bold.
}
\label{tab:cifar10}
\begin{center}
\begin{tabular}{c||c|c}    \toprule
\emph{Methods} & \emph{All}& \emph{400 per class}   \\\midrule
1 Layer K-means \cite{radford2015unsupervised}& 80.6\% & 63.7\% ($\pm$ 0.7\%) \\
3 Layer K-means Learned RF \cite{coates2011selecting}& 82.0\% & 70.7\%($\pm$ 0.7\%)  \\
View Invariant K-means \cite{hui2013direct}& 81.9\% & 72.6\%($\pm$ 0.7\%) \\
Examplar CNN \cite{dosovitskiydiscriminative}& 84.3\% & 77.4\%($\pm$ 0.2\%) \\
Conditional GAN \cite{mirza2014conditional} & 83.6\% & 75.5\%($\pm$ 0.4\%)\\
DCGAN \cite{radford2015unsupervised}& 82.8\% & 73.8\%($\pm$ 0.4\%) \\
Ladder Network \cite{rasmus2015semi} &  -- & 79.6\%($\pm$ 0.5\%)\\
CatGAN \cite{springenberg2015unsupervised} & -- & 80.4\%($\pm$ 0.4\%) \\
ALI \cite{dumoulin2016adversarially} & -- & 81.7\%\\
Improved GAN \cite{salimans2016improved} & -- & 81.4\%($\pm$ 2.3\%) \\\midrule
CLS-GAN & \bf 91.7\% & \bf 82.7\%($\pm$ 0.5\%) \\\bottomrule
\end{tabular}
\end{center}
\end{table}

The CIFAR dataset \cite{krizhevsky2009learning} consists of 50,000 training images and $10,000$ test images on ten image categories.
We tested the proposed CLS-GAN model with class labels as conditions.  In the supervised training, all labeled examples were used to train the CLS-GAN.

We also conducted experiments with $400$ labeled examples per class, which was a more challenging task as much fewer labeled examples were used for training. In this case, the remaining unlabeled examples were used to train the model in a semi-supervised fashion as discussed in Section~\ref{sec:clsgan}. In each mini-batch, the same number of labeled and unlabeled examples were used to update the model by stochastic gradient descent. The experiment results on this task were reported by averaging over ten subsets of labeled examples.

Both hyperparameters $\gamma$ and $\lambda$ were chosen via a five-fold cross-validation on the labeled examples from $\{0.25, 0.5, 1.0, 2.0\}$ and $\{0.5, 1.0, 2.0\}$ respectively. Once they were chosen, the model was trained with the chosen hyperparameters on the whole training set, and the performance was reported based on the results on the test set. As in the improved GAN, we also adopted the weight normalization and feature matching mechanisms for the sake of the fair comparison.


We compared the proposed model with the state-of-the-art methods in literature.  In particular, we compared with the conditional GAN \cite{mirza2014conditional} as well as the DCGAN \cite{radford2015unsupervised}. For the sake of fair comparison, the conditional GAN shared the same architecture as the CLS-GAN.
On the other hand, the DCGAN algorithm \cite{radford2015unsupervised} max-pooled
the discriminator's convolution features from all layers to $4\times 4$ grids as the image features, and a L2-SVM was then trained to classify images. The DCGAN was an unsupervised model which had shown competitive performance on generating photo-realistic images. Its feature representations were believed to reach the state-of-the-art performance in modeling images with no supervision.



We also compared with the other recently developed supervised and semi-supervised models in literature, including the baseline 1 Layer K-means feature extraction pipeline, a multi-layer extension of the baseline model (3 Layer K-means Learned RF \cite{coates2011selecting}), View Invariant K-means \cite{hui2013direct}, Examplar CNN \cite{dosovitskiydiscriminative}, Ladder Network \cite{rasmus2015semi}, as well as CatGAN \cite{springenberg2015unsupervised}.  In particular, among the compared semi-supervised algorithms, the improved GAN \cite{salimans2016improved} had recorded the best performance in literature.  Furthermore, we also compared with the ALI \cite{dumoulin2016adversarially} that extended the classic GAN by jointly generating data and inferring their representations, which achieved comparable performance to the Improved GAN. This pointed out an interesting direction to extend the CLS-GAN by directly inferring the data representation, and we will leave it in the future work.

Table~\ref{tab:cifar10} compares the experiment results, showing the CSL-GAN successfully outperformed the compared algorithms in both fully-supervised and semi-supervised settings.

\subsubsection{SVHN}

\begin{table}[t]
\caption{Classification errors on SVHN dataset with $1,000$ labeled examples. The best result is highlighted in bold.}
\label{tab:svhn}
\begin{center}
\begin{tabular}{c||c}    \toprule
\emph{Methods} & \emph{Error rate} \\\midrule
KNN \cite{radford2015unsupervised}& 77.93\%\\
TSVM \cite{radford2015unsupervised}& 66.55\%\\
M1+KNN \cite{kingma2014semi}& 65.63\%\\
M1+TSVM \cite{kingma2014semi}& 54.33\%\\
M1+M2 \cite{kingma2014semi}& 36.02\%\\\midrule
SWWAE w/o dropout \cite{zhao2015stacked}& 27.83\%\\
SWWAE with dropout \cite{zhao2015stacked}& 23.56\%\\
DCGAN \cite{radford2015unsupervised}& 22.48\%\\
Conditional GAN \cite{mirza2014conditional}& 21.85\%$\pm$0.38\%\\
Supervised CNN \cite{radford2015unsupervised}& 28.87\%\\
DGN \cite{kingma2014semi} &36.02\%$\pm$0.10\%\\
Virtual Adversarial \cite{miyato2015distributional} &24.63\%\\
Auxiliary DGN \cite{maaloe2016auxiliary} &22.86\%\\
Skip DGN \cite{maaloe2016auxiliary}&16.61\%$\pm$0.24\%\\
ALI \cite{dumoulin2016adversarially} & 7.3\%\\
Improved GAN \cite{salimans2016improved}&8.11\%$\pm$1.3\%\\\midrule
CLS-GAN & {\bf 5.98\%$\pm$ 0.27\%}\\\bottomrule
\end{tabular}
\vspace{-2mm}
\end{center}
\end{table}

The SVHN (i.e., Street View House Number) dataset \cite{netzer2011reading} contains $32\times 32$ color images of house numbers collected by Google Street View. They were roughly centered on a digit in a house number, and the objective is to recognize the digit.  The training set has $73,257$ digits while the test set consists of $26,032$.

To test the model, $1,000$ labeled digits were used to train the model, which are uniformly selected from ten digit classes, that is $100$ labeled examples per digit class. The remaining unlabeled examples were used as additional data to enhance the generative ability of CLS-GAN in semi-supervised fashion.
We expect a good generative model could produce additional examples to augment the training set.

We used the same experiment setup and network architecture for CIFAR-10 to train the LS-GAN on this dataset.  Table~\ref{tab:svhn} reports the result on the SVHN, and it shows that the LS-GAN performed the best among the compared algorithms.

\subsubsection{Analysis of Generated Images by CLS-GAN}

\begin{figure*}[t]
\centering
\subfigure[MNIST]{
\begin{minipage}{0.28\linewidth}
\begin{center}
	   \includegraphics[width=1.0\linewidth]{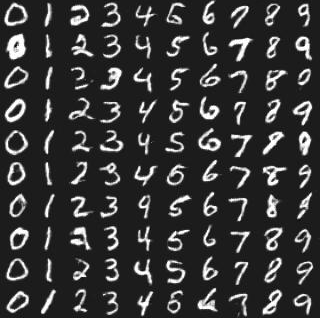}
\end{center}
\end{minipage}}\hspace{4mm}
\subfigure[CIFAR-10]{
\begin{minipage}{0.28\linewidth}
\begin{center}
	   \includegraphics[width=1.0\linewidth]{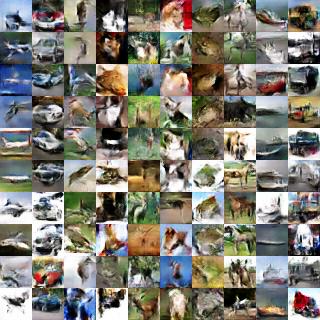}
\end{center}
\end{minipage}}\hspace{4mm}
\subfigure[SVHN]{
\begin{minipage}{0.28\linewidth}
\begin{center}
	   \includegraphics[width=1.0\linewidth]{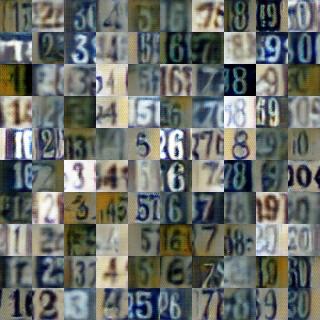}
\end{center}
\end{minipage}}
   \caption{Images generated by CLS-GAN for MNIST, CIFAR-10 and SVHN. Images in a column are generated for the same class. In particular, the generated images on CIFAR-10 are airplane, automobile, bird, cat, deer, dog, frog, horse, ship and truck from the leftmost to the rightmost column.}\label{fig:generation}
\label{Fig:dRNN}
\end{figure*}

Figure~\ref{fig:generation} illustrates the generated images by CLS-GAN for MNIST, CIFAR-10 and SVHN datasets. On each dataset, images in a column were generated for the same class. On the MNIST and the SVHN, both handwritten and street-view digits are quite legible. Both also cover many variants for each digit class. For example, the synthesized MNIST digits have various writing styles, rotations and sizes, and the generated SVHN digits have various lighting conditions, sizes and even different co-occurring digits in the cropped bounding boxes.  On the CIFAR-10 dataset, image classes can be recognized from the generated images although some visual details are missing. This is because the images in the CIFAR-10 dataset have very low resolution ($32\times 32$ pixels), and most details are even missing from input examples.

We also observe that if we set a small value to the hyperparameter $\lambda$, the generated images would become very similar to each other within each class. As illustrated in Figure~\ref{fig:collapse}, the images were generated by halving $\lambda$ used for generating images in Figure~\ref{fig:generation}.
A smaller $\lambda$ means a relatively large weight was placed on the first loss minimization term of (\ref{eq:theta1}), which tends to collapse generated images to a single mode as it aggressively minimizes their losses to train the generator. This is also consistent with Theorem~\ref{thm4} where the density of generated samples with a smaller $\lambda$ could have a larger deviation from the underlying density.
One should avoid the collapse of trained generator since
diversifying generated images can improve the classification performance of the CLS-GAN by revealing more intra-class variations.  This will help improve the model's generalization ability as these variations could appear in future images.

\begin{figure*}[t]
\centering
\subfigure[MNIST]{
\begin{minipage}{0.28\linewidth}
\begin{center}
	   \includegraphics[width=1.0\linewidth]{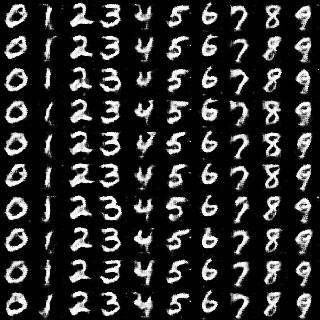}
\end{center}
\end{minipage}}\hspace{4mm}
\subfigure[CIFAR-10]{
\begin{minipage}{0.28\linewidth}
\begin{center}
	   \includegraphics[width=1.0\linewidth]{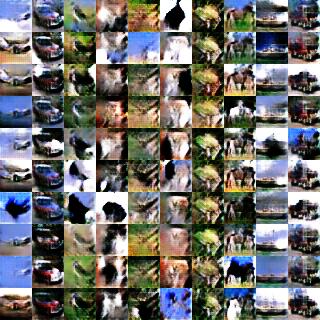}
\end{center}
\end{minipage}}\hspace{4mm}
\subfigure[SVHN]{
\begin{minipage}{0.28\linewidth}
\begin{center}
	   \includegraphics[width=1.0\linewidth]{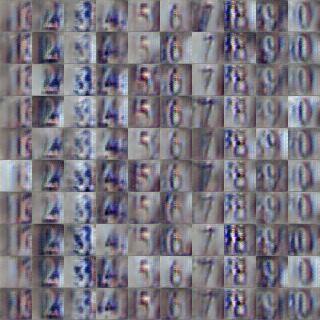}
\end{center}
\end{minipage}}
   \caption{Illustration of generated images that are collapsed to a single mode of the underlying image density on MNIST, CIFAR-10 and SVHN.}\label{fig:collapse}
\end{figure*}

However, one should also avoid setting too large value to $\lambda$.  Otherwise, the role of the first loss minimization term could be underestimated, which can also adversely affect the classification results without reducing the training loss to a satisfactory level.  Therefore, we choose a proper value for $\lambda$ by cross-validation on the training set in the experiments.

In brief, the comparison between Figure~\ref{fig:generation} and Figure~\ref{fig:collapse} reveals a trade-off between image generation quality and classification accuracy through the hyperparameter $\lambda$. Such a trade-off is intuitive: while a classification task usually focuses on learning class-invariant representations that do not change within a class, image generation should be able to capture many variant factors (e.g., lighting conditions, viewing angles, and object poses) so that it could diversify generated samples for each class. Although diversified examples can augment training dataset, it comes at a cost of trading class-invariance for modeling variant generation factors. Perhaps, this is an intrinsic dilemma between supervised learning and data generation that is worth more theoretical and empirical studies in future.

\subsection{Evaluation of Generalization Performances}\label{sec:eval_gen}

Most of existing metrics like Inception Score \cite{salimans2016improved} for evaluating GAN models focus on comparing the qualities and diversities of their generated images. However, even though a GAN model can produce diverse and high quality images with no collapsed generators, it is still unknown if the model can generate {\em unseen} images out of given examples, or simply memorizing existing ones.  
While one of our main pursuits in this paper is a generalizable LS-GAN, we were motivated to propose the following Minimum Reconstruction Error (MRE) to compare its generalizability with various GANs.



Specifically, for an unseen test image $\mathbf x$, we aim to find an input noise $\mathbf z$ that can best reconstruct $\mathbf x$ with the smallest error, i.e.,
$
\min_\mathbf z \|\mathbf x-G(\mathbf z) \|_1,
$
where $G$ is the GAN generator under evaluation. Obviously, if $G$ is adequate to produce new images, it should have a small reconstruction error on a separate test set that has not been used in training the model.

We assessed the GAN's generalizability on CIFAR-10 and tiny ImageNet datasets. On CIFAR-10, we split the dataset into 50\% training examples, 25\% validation examples and 25\% test examples; the tiny ImageNet was split into training, validation and test sets in a ratio of 10:1:1. For a fair comparison, all the hyperparameters, including the number of epochs, were chosen based on the average MREs on the validation set, and the test MREs were reported for comparison. The optimal $\mathbf z$'s were iteratively updated on the validation and test sets by descending the gradient of the reconstruction errors.


\begin{figure*}
\centering
\subfigure[CIFAR-10]{
\begin{minipage}{0.49\linewidth}
\begin{center}
	   \includegraphics[width=0.98\linewidth]{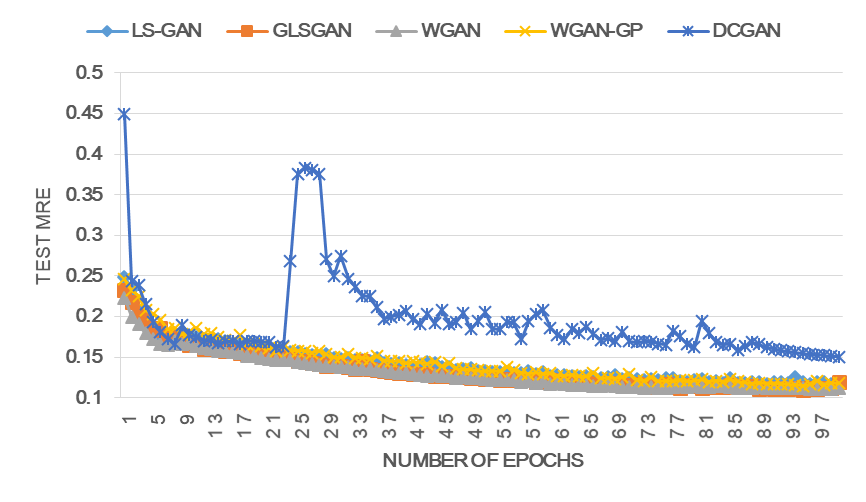}\vspace{1mm}
\end{center}
\end{minipage}}
\subfigure[tiny ImageNet]{
\begin{minipage}{0.49\linewidth}
\begin{center} \includegraphics[width=0.98\linewidth]{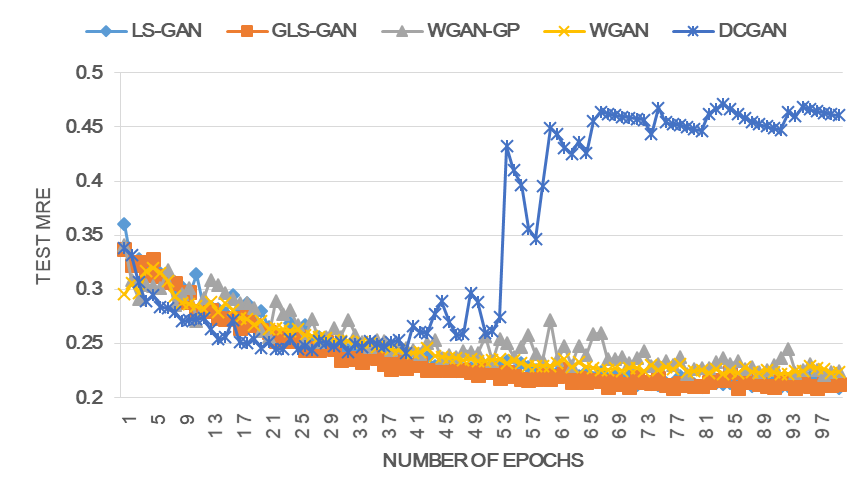}\vspace{1mm}
\end{center}
\end{minipage}}
   \caption{The change of test MREs on CIFAR-10 and tiny ImageNet over epochs. Image pixels were scaled to $[-1,1]$ to compute the MREs. }\label{fig:cifar10_mre}
\end{figure*}

\begin{figure*}[h]
\centering
\subfigure[Original test images]{
\begin{minipage}{0.25\linewidth}
\begin{center}
	   \includegraphics[width=1.0\linewidth]{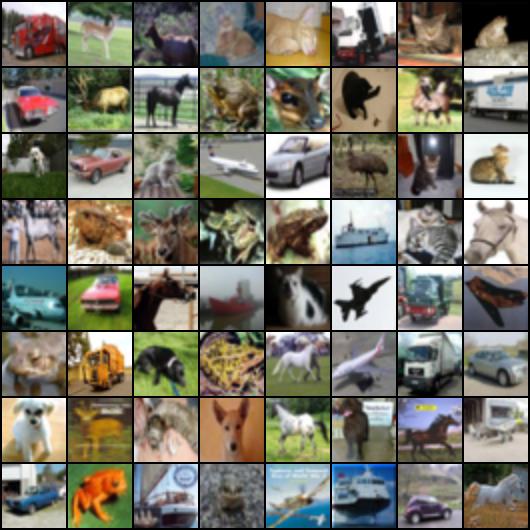}\vspace{1mm}
\end{center}
\end{minipage}}
\subfigure[LS-GAN(0.1166)]{
\begin{minipage}{0.25\linewidth}
\begin{center}
	   \includegraphics[width=1.0\linewidth]{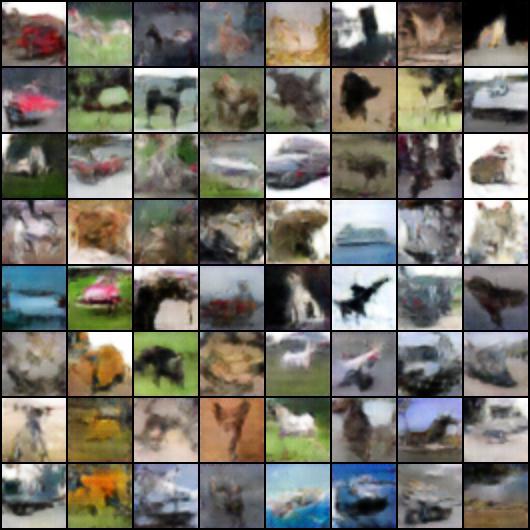}\vspace{1mm}
\end{center}
\end{minipage}}
\subfigure[\bf GLS-GAN(0.1089)]{
\begin{minipage}{0.25\linewidth}
\begin{center}
	   \includegraphics[width=1.0\linewidth]{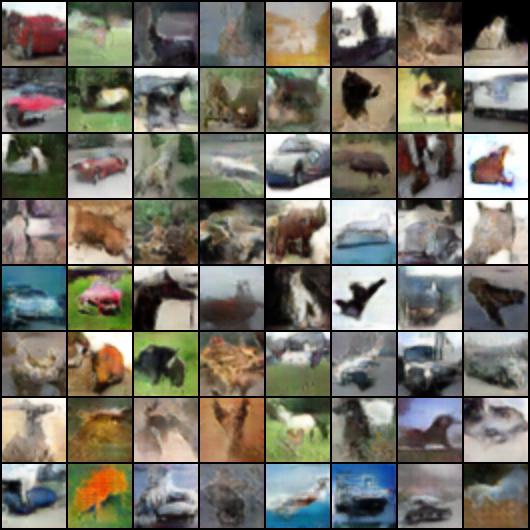}\vspace{1mm}
\end{center}
\end{minipage}}\\
\subfigure[WGAN-GP(0.1149)]{
\begin{minipage}{0.25\linewidth}
\begin{center}
	   \includegraphics[width=1.0\linewidth]{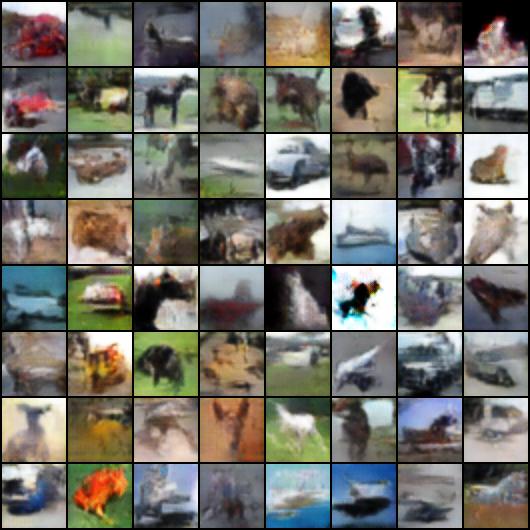}\vspace{1mm}
\end{center}
\end{minipage}}
\subfigure[WGAN(0.1109)]{
\begin{minipage}{0.25\linewidth}
\begin{center}
	   \includegraphics[width=1.0\linewidth]{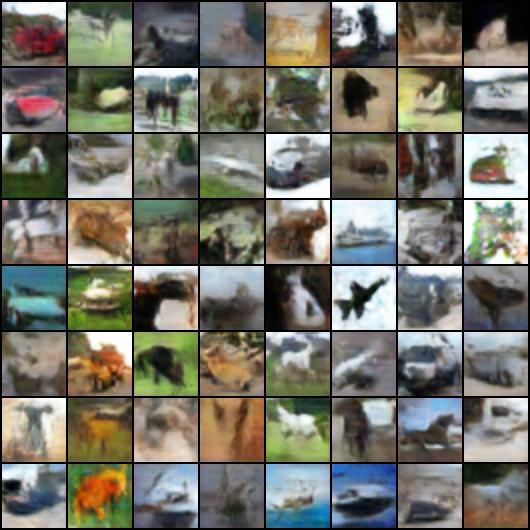}\vspace{1mm}
\end{center}
\end{minipage}}
\subfigure[DCGAN(0.1506)]{
\begin{minipage}{0.25\linewidth}
\begin{center}
	   \includegraphics[width=1.0\linewidth]{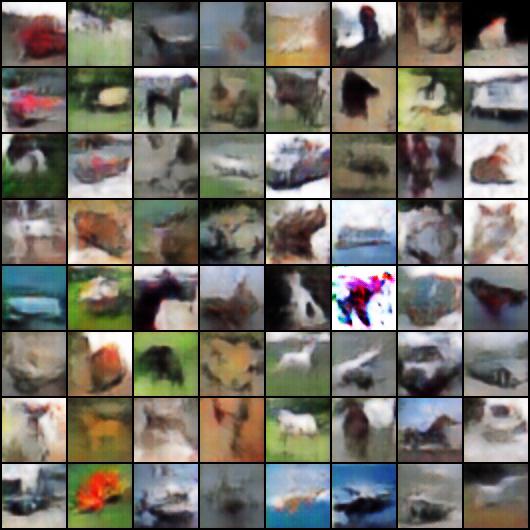}\vspace{1mm}
\end{center}
\end{minipage}}
   \caption{The figure illustrates the images reconstructed by various GANs on CIFAR-10 with their MREs on the test set in the parentheses. }\label{fig:reconstructed_cifar10}
\end{figure*}

\begin{figure*}[h]
\centering
\subfigure[Original test images]{
\begin{minipage}{0.25\linewidth}
\begin{center}
	   \includegraphics[width=1.0\linewidth]{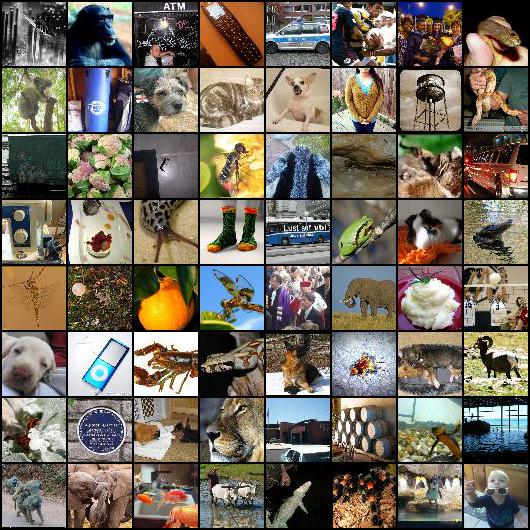}\vspace{1mm}
\end{center}
\end{minipage}}
\subfigure[LS-GAN(0.2093)]{
\begin{minipage}{0.25\linewidth}
\begin{center}
	   \includegraphics[width=1.0\linewidth]{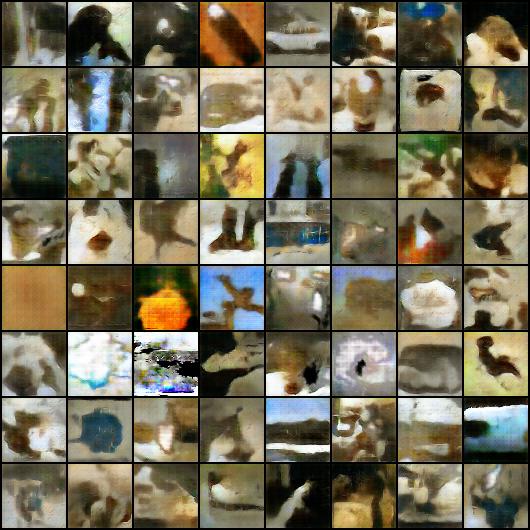}\vspace{1mm}
\end{center}
\end{minipage}}
\subfigure[\bf GLS-GAN(0.2085)]{
\begin{minipage}{0.25\linewidth}
\begin{center}
	   \includegraphics[width=1.0\linewidth]{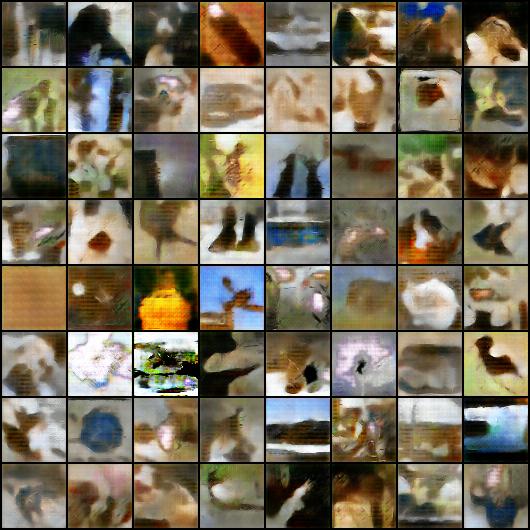}\vspace{1mm}
\end{center}
\end{minipage}}\\
\subfigure[WGAN-GP(0.2210)]{
\begin{minipage}{0.25\linewidth}
\begin{center}
	   \includegraphics[width=1.0\linewidth]{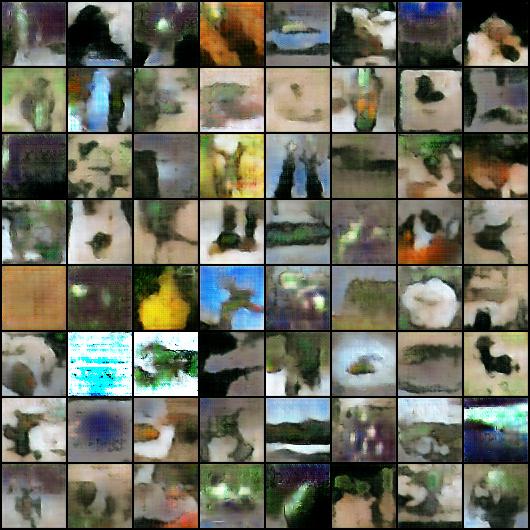}\vspace{1mm}
\end{center}
\end{minipage}}
\subfigure[WGAN(0.2219)]{
\begin{minipage}{0.25\linewidth}
\begin{center}
	   \includegraphics[width=1.0\linewidth]{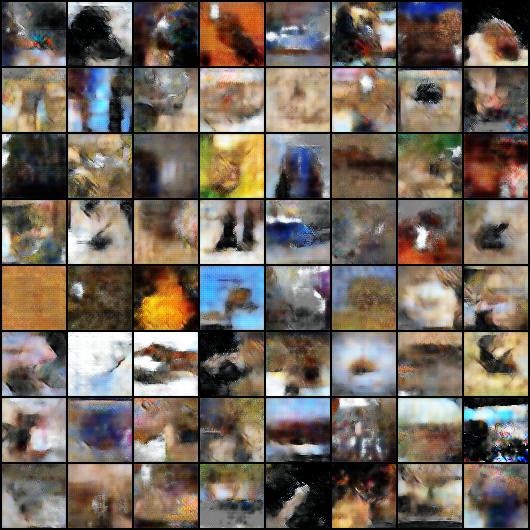}\vspace{1mm}
\end{center}
\end{minipage}}
\subfigure[DCGAN(0.2413)]{
\begin{minipage}{0.25\linewidth}
\begin{center}
	   \includegraphics[width=1.0\linewidth]{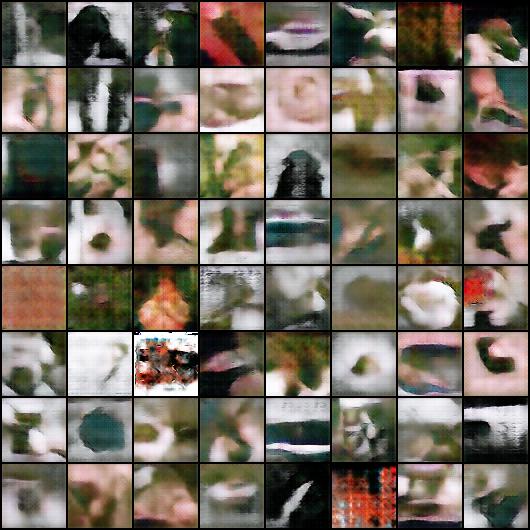}\vspace{1mm}
\end{center}
\end{minipage}}
   \caption{The figure illustrates the images reconstructed by various GANs on tiny ImageNet with their MREs on the test set in the parentheses.}\label{fig:reconstructed_cifar10}
\end{figure*}

In Figure~\ref{fig:cifar10_mre}, we compare the test MREs over $100$ epochs by LS-GAN, GLS-GAN, WGAN \cite{wgan17}, WGAN-GP \cite{gulrajani2017improved} and DCGAN \cite{radford2015unsupervised} on CIFAR-10 respectively.
For the sake of a fair comparison, all models were trained with the network architecture used in \cite{radford2015unsupervised}.
The result clearly shows the regularized models, including GLS-GAN, LS-GAN, WGAN-GP and WGAN, have apparently better generalization performances than the unregularized DCGAN based on the classic GAN model. On CIFAR-10, the test MRE was reduced from $0.1506$ by DCGAN to as small as $0.1109$ and $0.1089$ by WGAN and GLS-GAN respectively; on tiny ImageNet, the GLS-GAN reaches the smallest test MRE of $0.2085$ among all compared regularized and unregularized GANs.

In addition, the DCGAN exhibited fluctuating MREs on the CIFAR-10, while the regularized models steadily decreased the MREs over epochs. This implies regularized GANs have more stable training than the classic GAN. 


We illustrate some examples of reconstructed images by different GANs on the test set along with their test MREs in Figure~\ref{fig:reconstructed_cifar10}.
The results show the GLS-GAN achieved the smallest test MRE of 0.1089 and 0.2085 with a LeakyReLU cost function of slope $0.01$ and $0.5$ on CIFAR-10 and tiny ImageNet, followed by the other regularized GAN models.
This is not a surprising result since it has been shown in Section~\ref{sec:glsgan} that the other regularized GANs such as LS-GAN and WGAN are only special cases of the GLS-GAN model that covers larger family of models. Here we only considered LeakyReLU as the cost function for GLS-GAN. Of course, there exist many more cost functions satisfying the two conditions in Section~\ref{sec:glsgan}  to expand the family of regularized GANs, which should have potentials of yielding even better generalization performances.

\section{Conclusions}\label{sec:concl}
In this paper, we present a novel Loss-Sensitive GAN (LS-GAN) approach to generate samples from a data distribution.  The LS-GAN learns a loss function to distinguish between generated and real samples, where the loss of a real sample should be smaller by a margin than that of a generated sample.
Our theoretical analysis shows the distributional consistency between the real and generated samples based on the Lipschitz regularity.
This no longer needs a non-parametric discriminator with infinite modeling ability in the classic GAN, allowing us to search for the optimal loss function in a smaller functional space with a bounded Lipschitz constant.
Moreover, we prove the generalizability of LS-GAN by showing its required number of training examples is polynomial in its complexity. This suggests the generalization performance can be improved by penalizing the Lipschitz constants (via their gradient surrogates) of the loss function to reduce the sample complexity.
Furthermore, our non-parametric analysis of the optimal loss function shows its lower and upper bounds are cone-shaped with non-vanishing gradient almost everywhere, implying the generator can be continuously updated even if the loss function is over-trained.
Finally, we extend the LS-GAN to a Conditional LS-GAN (CLS-GAN) for semi-supervised tasks, and demonstrate it reaches competitive performances on both image generation and classification tasks.


\bibliographystyle{spmpsci}      
\bibliography{LSGAN}   


\appendix

\section{Proof of Lemma~\ref{lem1}}\label{proof_a}
To prove Lemma~\ref{lem1}, we need the following lemma.

\begin{lemma}\label{lem3}
For two probability densities $p(\mathbf x)$ and $q(\mathbf x)$, if $p(\mathbf x)\geq \eta q(\mathbf x)$ almost everywhere, we have
$$
\int_\mathbf x |p(\mathbf x)-q(\mathbf x)|d\mathbf x \leq \dfrac{2(1-\eta)}{\eta}
$$
for $\eta\in(0,1]$.
\end{lemma}
\begin{proof}
We have the following equalities and inequalities:
\begin{equation}
\begin{aligned}
&\int_{\mathbf x} |p(\mathbf x)-q(\mathbf x)|d\mathbf x
=\int_\mathbf x \1{p(\mathbf x)\geq q(\mathbf x)} (p(\mathbf x)-q(\mathbf x))d\mathbf x
+\int_\mathbf x \1{p(\mathbf x)<q(\mathbf x)} (q(\mathbf x)-p(\mathbf x))d\mathbf x\\
&=\int_\mathbf x (1-\1{p(\mathbf x)< q(\mathbf x)}) (p(\mathbf x)-q(\mathbf x))d\mathbf x
+\int_\mathbf x \1{p(\mathbf x)<q(\mathbf x)} (q(\mathbf x)-p(\mathbf x))d\mathbf x\\
&=2\int_\mathbf x \1{p(\mathbf x)< q(\mathbf x)} (q(\mathbf x)-p(\mathbf x))d\mathbf x
\leq 2(\dfrac{1}{\eta}-1)\int_\mathbf x \1{p(\mathbf x)< q(\mathbf x)}  p(\mathbf x) d\mathbf x
\leq \dfrac{2(1-\eta)}{\eta}
\end{aligned}
\end{equation}
This completes the proof.
\end{proof}

Now we can prove Lemma~\ref{lem1}.
\begin{proof}
Suppose $(\theta^*,\phi^*)$ is a Nash equilibrium for the problem (\ref{eq:theta}) and (\ref{eq:phi}).

Then, on one hand, we have
\begin{equation}\label{eq:lower1}
\begin{aligned}
&S(\theta^*,\phi^*)\geq\mathop\mathbb E\limits_{\mathbf x\sim P_{data}(\mathbf x)} L_{\theta^*}(\mathbf x)
+ \lambda\mathop\mathbb E\limits_{\substack{\mathbf x\sim P_{data}(\mathbf x)\\ \mathbf z_G\sim P_{G^*}(\mathbf z_G)}} \big(\Delta(\mathbf x, \mathbf z_G)
+L_{\theta^*}(\mathbf x)-L_{\theta^*}(\mathbf z_G)\big)\\
&=\int_\mathbf x P_{data}(\mathbf x) L_{\theta^*}(\mathbf x) d\mathbf x + \lambda \mathop \mathbb E\limits_{\substack{\mathbf x\sim P_{data}(\mathbf x) \\ \mathbf z_G\sim P_{G^*}(\mathbf z_G)}} \Delta(\mathbf x, \mathbf z_G) \\
&+\lambda\int_\mathbf x P_{data}(\mathbf x) L_{\theta^*}(\mathbf x) d\mathbf x - \lambda\int_{\mathbf z_G} P_{G^*}(\mathbf z_G) L_{\theta^*}(\mathbf z_G) d \mathbf z_G\\
&=\int_\mathbf x \big((1+\lambda)P_{data}(\mathbf x)- \lambda P_{G^*}(\mathbf x)\big)L_{\theta^*}(\mathbf x) d\mathbf x +\lambda \mathop \mathbb E\limits_{\substack{\mathbf x\sim P_{data}(\mathbf x) \\ \mathbf z_G\sim P_{G^*}(\mathbf z_G)}}\Delta(\mathbf x, \mathbf z_G)
\end{aligned}
\end{equation}
where the first inequality follows from $(a)_+\geq a$.

We also have $T(\theta^*,\phi^*)\leq T(\theta^*,\phi)$ for any $G_\phi$ as $\phi^*$ minimizes $T(\theta^*,\phi)$. In particular, we can replace $P_G(\mathbf x)$ in $T(\theta^*,\phi)$ with $P_{data}(\mathbf x)$, which yields
$$
\mathop\int\limits_{\mathbf x}L_{\theta^*}(\mathbf x)P_{G^*}(\mathbf x)d\mathbf x \leq
\mathop\int\limits_{\mathbf x}L_{\theta^*}(\mathbf x)P_{data}(\mathbf x)d\mathbf x.
$$

Applying this inequality into (\ref{eq:lower1}) leads to
\begin{equation}\label{eq:lower2}
\begin{aligned}
S(\theta^*,\phi^*)
\geq\int_\mathbf x P_{data}(\mathbf x)L_{\theta^*}(\mathbf x) d\mathbf x
+\lambda \mathop \mathbb E\limits_{\substack{\mathbf x\sim P_{data}(\mathbf x) \\ \mathbf z_G\sim P_{G^*}(\mathbf z_G)}}\Delta(\mathbf x, \mathbf z_G)
\geq\lambda \mathop \mathbb E\limits_{\substack{\mathbf x\sim P_{data}(\mathbf x) \\ \mathbf z_G\sim P_{G^*}(\mathbf z_G)}}\Delta(\mathbf x, \mathbf z_G)
\end{aligned}
\end{equation}
where the last inequality follows as $L_\theta(\mathbf x)$ is nonnegative.


On the other hand, consider a particular loss function
\begin{equation}\label{eq:L}
L_{\theta_0}(\mathbf x)=\alpha\big(-(1+\lambda) P_{data}(\mathbf x)+\lambda P_{G^*}(\mathbf x)\big)_+
\end{equation}
When $\alpha$ is a sufficiently small positive coefficient, $L_{\theta_0}(\mathbf x)$ is a nonexpansive function (i.e., a function with Lipschitz constant no larger than $1$.). This follows from the assumption that  $P_{data}$ and $P_G$ are Lipschitz.  In this case, we have
\begin{equation}\label{eq:nonex}
\Delta(\mathbf x, \mathbf z_G)
+L_{\theta_0}(\mathbf x)-L_{\theta_0}(\mathbf z_G) \geq 0
\end{equation}

By placing this $L_{\theta_0}(\mathbf x)$ into $S(\theta,\phi^*)$, one can show that
\[
\begin{aligned}
&S(\theta_0,\phi^*)
=\int_\mathbf x \big((1+\lambda)P_{data}(\mathbf x)- \lambda P_{G^*}(\mathbf x)\big)L_{\theta_0}(\mathbf x) d\mathbf x
+\lambda \mathop \mathbb E\limits_{\substack{\mathbf x\sim P_{data}(\mathbf x) \\ \mathbf z_G\sim P_{G^*}(\mathbf z_G)}}\Delta(\mathbf x, \mathbf z_G)\\
&=-\alpha\mathop\int\limits_{\mathbf x} \big(-(1+\lambda)P_{data}(\mathbf x)+\lambda P_{G^*}(\mathbf x)\big)_+^2 d\mathbf x
+\lambda \mathop \mathbb E\limits_{\substack{\mathbf x\sim P_{data}(\mathbf x) \\ \mathbf z_G\sim P_{G^*}(\mathbf z_G)}}\Delta(\mathbf x, \mathbf z_G)\\
\end{aligned}
\]
where the first equality uses Eq.~(\ref{eq:nonex}), and the second equality is obtained by substituting  $L_{\theta_0}(\mathbf x)$ in Eq.~(\ref{eq:L}) into the equation.

Assuming that $(1+\lambda)P_{data}(\mathbf x)- \lambda P_{G^*}(\mathbf x)<0$ on a set of nonzero measure, the above equation would be strictly upper bounded by $\lambda \mathop \mathbb E\limits_{\substack{\mathbf x\sim P_{data}(\mathbf x) \\ \mathbf z_G\sim P_{G^*}(\mathbf z_G)}}\Delta(\mathbf x, \mathbf z_G)$ and we have
\begin{equation}
\begin{aligned}
S(\theta^*,\phi^*)\leq S(\theta_0,\phi^*)<\lambda \mathop \mathbb E\limits_{\substack{\mathbf x\sim P_{data}(\mathbf x) \\ \mathbf z_G\sim P_{G^*}(\mathbf z_G)}}\Delta(\mathbf x, \mathbf z_G)
\end{aligned}
\end{equation}

This results in a contradiction with Eq.~(\ref{eq:lower2}).
Therefore, we must have
\begin{equation}\label{eq:cond}
P_{data}(\mathbf x) \geq \dfrac{\lambda}{1+\lambda}P_{G^*}(\mathbf x)
\end{equation}
for almost everywhere.  By Lemma~\ref{lem3}, we have
$$
\int_{\mathbf x}|P_{data}(\mathbf x)-P_{G^*}(\mathbf x)|d\mathbf x \leq \dfrac{2}{\lambda}
$$

Let $\lambda\rightarrow +\infty$, this leads to
$$
\int_{\mathbf x}|P_{data}(\mathbf x)-P_{G^*}(\mathbf x)|d\mathbf x \rightarrow 0
$$
This proves that $P_{G^*}(\mathbf x)$ converges to $P_{data}(\mathbf x)$ as $\lambda\rightarrow +\infty$.
\end{proof}

\section{Proof of Lemma~\ref{lem:wgan}}\label{appendixC}
\begin{proof}
Suppose a pair of $(f_w^*, g_\phi^*)$ jointly solve the WGAN problem.

Then, on one hand, we have
\begin{align}\label{eq:lem_wgan1}
U(f_w^*,g_\phi^*)=\int_x f_w^*(\mathbf x) P_{data}(\mathbf x)d\mathbf x - \int_x f_w^* (\mathbf x)P_{g_\phi^*}(\mathbf x)d\mathbf x\leq 0
\end{align}
where the inequality follows from $V(f_w^*,g_\phi^*)\geq V(f_w^*,g_\phi)$ by replacing $P_{g_\phi}(\mathbf x)$ with $P_{data}(\mathbf x)$.

Consider a particular $f_w(\mathbf x)\triangleq \alpha(P_{data}(\mathbf x)-P_{g_\phi^*}(\mathbf x))_+$.  Since $P_{data}(\mathbf x)$ and $P_{g_\phi^*}$ are Lipschitz by assumption, when $\alpha$ is sufficiently small, it can be shown that $f_w(\mathbf x)\in\mathcal L_1$.

Substituting this $f_w$ into $U(f_w, g_\phi^*)$, we get
$$
U(f_w,g_\phi^*)=\alpha \int_\mathbf x (P_{data}(\mathbf x)-P_{g_\phi^*}(\mathbf x))_+^2 d\mathbf x
$$

Let us assume $P_{data}(\mathbf x) > P_{g_\phi^*}(\mathbf x)$ on a set of nonzero measure, we would have
$$
U(f_w^*,g_\phi^*)\geq U(f_w,g_\phi^*) > 0
$$
This leads to a contradiction with (\ref{eq:lem_wgan1}), so we must have
$$
P_{data}(\mathbf x) \leq P_{g_\phi^*}(\mathbf x)
$$
almost everywhere.

Hence, by Lemma~\ref{lem3}, we prove the conclusion that
$$
\int_{\mathbf x}|P_{data}(\mathbf x)-P_{g_\phi^*}(\mathbf x)|d\mathbf x=0.
$$
\end{proof}

\section{Proof of Theorem \ref{thm:generalization}}\label{sec:gen_proof}
For simplicity, throughout this section, we disregard the first loss minimization term in $S(\theta,\phi^*)$ and $S_m(\theta,\phi^*)$, since the role of the first term would vanish as $\lambda$ goes to $+\infty$. However, even if it is involved, the following proof still holds with only some minor changes.

To prove Theorem~\ref{thm:generalization}, we need the following lemma.
\begin{lemma}
For all loss functions $L_\theta$, with at least the probability of $1-\eta$, we have
$$
|S_m(\theta,\phi^*)-S(\theta,\phi^*)|\leq \varepsilon
$$
when the number of samples $$
m\geq\dfrac{C B_\Delta^2(\kappa+1)^2 \big(N \log(\kappa_L N/\varepsilon)+\log(1/\eta)\big)}{\varepsilon^2}$$
with a sufficiently large constant $C$.
\end{lemma}
The proof of this lemma needs to apply the McDiarmid's inequality and the fact that $(\cdot)_+$ is an 1-Lipschitz
to bound the difference $|S_m(\theta,\phi^*)-S(\theta,\phi^*)|$ for a loss function.  Then, to get the union bound over all loss functions, a standard $\epsilon$-net \cite{arora2017generalization} will be constructed to yield finite points that are dense enough to cover the parameter space of the loss functions. The proof details are given below.
\begin{proof}
For a loss function $L_\theta$, we compute $S_m(\theta,\phi^*)$ over a set of $m$ samples $\{(\mathbf x_i, {\mathbf z_G}_i)|1\leq i\leq m\}$ drawn from $P_{data}$ and $P_{G^*}$ respectively.

To apply the McDiarmid's inequality, we need to bound the change of this function when a sample is changed.  Denote by $S^i_m(\theta,\phi^*)$ when the $j$th sample is replaced with $\mathbf x'_i$ and ${\mathbf z'_G}_i$. Then we have
\[
\begin{aligned}
&|S_m(\theta,\phi^*)-S^i_m(\theta,\phi^*)|\\
&=\dfrac{1}{m}|\big(\Delta(\mathbf x_i, {\mathbf z_G}_i)+L_\theta(\mathbf x_i)-L_\theta({\mathbf z_G}_i)\big)_+
-\big(\Delta(\mathbf x'_i, {\mathbf z'_G}_i)+L_\theta(\mathbf x'_i)-L_\theta({\mathbf z'_G}_i)\big)_+|\\
&\leq \dfrac{1}{m}|\Delta(\mathbf x_i, {\mathbf z_G}_i)-\Delta(\mathbf x'_i, {\mathbf z'_G}_i)|
+\dfrac{1}{m}|L_\theta(\mathbf x_i)-L_\theta(\mathbf x'_i)|
+\dfrac{1}{m}|L_\theta({\mathbf z_G}_i)-L_\theta({\mathbf z'_G}_i)|\\
&\leq \dfrac{1}{m}\big(2B_\Delta+\kappa \Delta(\mathbf x_i,\mathbf x'_i) + \kappa \Delta({\mathbf z_G}_i,{\mathbf z'_G}_i)\big)
\leq \dfrac{2}{m}(1+\kappa)B_\Delta
\end{aligned}
\]
where the first inequality uses the fact that $(\cdot)_+$ is $1$-Lipschitz, the second inequality follows from that
$\Delta(\mathbf x, \mathbf z_G)$ is bounded by $B_\Delta$ and $L_\theta(\mathbf x)$ is $\kappa$-Lipschitz in $\mathbf x$.

Now we can apply the McDiarmid's inequality. Noting that
$$S(\theta,\phi^*)=\mathop\mathbb E\limits_{\substack{\mathbf x_i\sim P_{data}\\ {\mathbf z_G}_i\sim  P_G\\i=1,\cdots,m}}S_m(\theta,\phi^*),$$
we have
\begin{equation}\label{eq:mcdiarmid}
\small
P(|S_m(\theta,\phi^*)-S(\theta,\phi^*)|\geq \varepsilon/2)\leq 2\exp(-\dfrac{\varepsilon^2 m}{8(1+\kappa)^2B_\Delta^2})
\end{equation}
The above bound applies to a single loss function $L_\theta$.  To get the union bound, we consider a $\varepsilon/8\kappa_L$-net $\mathcal N$, i.e., for any $L_\theta$, there is a $\theta'\in\mathcal N$ in this net so that $\|\theta-\theta'\|\leq \varepsilon/8\kappa_L$. This standard net can be constructed to contain finite loss functions such that $|\mathcal N|\leq O(N\log(\kappa_L N/\varepsilon))$, where $N$ is the number of parameters in a loss function. Note that we implicitly assume the parameter space of the loss function is bounded so we can construct such a net containing finite points here.

Therefore, we have the following union bound for all $\theta\in\mathcal N$ that, with probability $1-\eta$,
$$
|S_m(\theta,\phi^*)-S(\theta,\phi^*)|\leq \dfrac{\varepsilon}{2}
$$
when $m\geq\dfrac{C B_\Delta^2(\kappa+1)^2 \big(N \log(\kappa_L N/\varepsilon)+\log(1/\eta)\big)}{\varepsilon^2}$.

The last step is to obtain the union bound for all loss functions beyond $\mathcal N$. To show that, we consider the following inequality
\[
\begin{aligned}
&|S(\theta,\phi^*)-S(\theta',\phi^*)|\\
&=|\mathop\mathbb E\limits_{\substack{\mathbf x\sim P_{data}\\ \mathbf z_G\sim P_G}} \big(\Delta(\mathbf x, \mathbf z_G)+ L_{\theta}(\mathbf x)- L_{\theta}(\mathbf z_G)\big)_+
-\mathop\mathbb E\limits_{\substack{\mathbf x\sim P_{data}\\ \mathbf z_G\sim P_G}} \big(\Delta(\mathbf x, \mathbf z_G)+ L_{\theta'}(\mathbf x)- L_{\theta'}(\mathbf z_G)\big)_+|\\
&\leq \mathop\mathbb E\limits_{\substack{\mathbf x\sim P_{data}}}|L_{\theta}(\mathbf x)-L_{\theta'}(\mathbf x)|+\mathop\mathbb E\limits_{\substack{ \mathbf z_G\sim P_G}}|L_{\theta}(\mathbf z_G)-L_{\theta'}(\mathbf z_G)|\\
&\leq 2\kappa_L\|\theta-\theta'\|
\end{aligned}
\]
where the first inequality uses that fact that $(\cdot)_+$ is $1$-Lipschitz again, and the second inequality follows from that $L_\theta$ is $\kappa_L$-Lipschitz in $\theta$.
Similarly, we can also show that
\[
\begin{aligned}
|S_m(\theta,\phi^*)-S_m(\theta',\phi^*)|
\leq 2\kappa_L\|\theta-\theta'\|
\end{aligned}
\]
Now we can derive the union bound over all loss functions. For any $\theta$, by construction we can find a $\theta'\in\mathcal N$ such that $\|\theta-\theta'\|\leq\varepsilon/8\kappa_L$.
Then, with probability $1-\eta$, we have
\[
\begin{aligned}
&|S_m(\theta,\phi^*)-S(\theta,\phi^*)|\\
&\leq|S_m(\theta,\phi^*)-S_m(\theta',\phi^*)|
+|S_m(\theta',\phi^*)-S(\theta',\phi^*)|+|S(\theta',\phi^*)-S(\theta,\phi^*)|\\
&\leq 2\kappa_L\|\theta-\theta'\|+\dfrac{\varepsilon}{2}+2\kappa_L\|\theta-\theta'\|\\
&\leq\dfrac{\varepsilon}{4}+\dfrac{\varepsilon}{2}+\dfrac{\varepsilon}{4}=\varepsilon
\end{aligned}
\]
This proves the lemma.
\end{proof}

Now we can prove Theorem~\ref{thm:generalization}.
\begin{proof}
First let us bound $S_m-S$.  Consider $L_{\theta^*}$ that minimizes $S(\theta,\phi^*)$. Then with probability $1-\eta$, when $m\geq\dfrac{C N B_\Delta^2(\kappa+1)^2\log(\kappa_L N/\eta\varepsilon)}{\varepsilon^2}$, we have
\[
\begin{aligned}
S_m - S \leq S_m(\theta^*,\phi^*) - S(\theta^*,\phi^*)\leq\varepsilon
\end{aligned}
\]
where the first inequality follows from the inequality $S_m\leq S_m(\theta^*,\phi^*)$ as $\theta^*$ may not minimize $S_m$, and the second inequality is a direct application of the above lemma.
Similarly, we can prove the other direction. With probability $1-\eta$, we have
\[
\begin{aligned}
S - S_m \geq S(\theta^*,\phi^*)- S_m(\theta^*,\phi^*) \geq\varepsilon
\end{aligned}
\]


Finally, a more rigourous discussion about the generalizability should consider that $G_{\phi^*}$ is updated iteratively. Therefore we have a sequence of $G_{\phi^*}^{(t)}$ generated over $T$ iterations for $t=1,\cdots,T$.  Thus, a union bound over all generators should be considered in (\ref{eq:mcdiarmid}), and this makes the required number of training examples $m$ become
$$
m\geq\dfrac{C B_\Delta^2(\kappa+1)^2 \big(N \log(\kappa_L N/\varepsilon)+\log(T/\eta)\big)}{\varepsilon^2}.
$$
However, the iteration number $T$ is usually much smaller than the model size $N$ (which is often hundreds of thousands), and thus this factor will not affect the above lower bound of $m$.
\end{proof}

\section{Proof of Theorem~\ref{thm_nonparam} and Corollary~\ref{cor2}}\label{proof_b}
We prove Theorem~\ref{thm_nonparam} as follows.

\begin{proof}
First, the existence of a minimizer follows from the fact that the functions in $\mathcal F_\kappa$ form a compact set, and the objective function is convex.

To prove the minimizer has the two forms in (\ref{eq:param}), for each $L_\theta \in \mathcal F_\kappa$, let us consider
\[
\begin{aligned}
&\widehat L_{\theta}(\mathbf x) = \max_{1\leq i\leq n+m}\big\{\big(L_\theta(\mathbf x^{(i)})-\kappa\Delta(\mathbf x,\mathbf x^{(i)})\big)_+\big\},\\
&\widetilde L_{\theta}(\mathbf x) = \min_{1\leq i\leq n+m}\big\{L_\theta(\mathbf x^{(i)})+\kappa\Delta(\mathbf x,\mathbf x^{(i)})\}
\end{aligned}
\]

It is not hard to verify that $\widehat L_\theta(\mathbf x^{(i)})=L_\theta(\mathbf x^{(i)})$ and $\widetilde L_\theta(\mathbf x^{(i)})=L_\theta(\mathbf x^{(i)})$ for $1\leq i\leq n+m$.

Indeed, by noting that $L_{\theta}$ has its Lipschitz constant bounded by $\kappa$, we have $L_{\theta}(\mathbf x^{(j)})-L_{\theta}(\mathbf x^{(i)})\leq \kappa \Delta(\mathbf x^{(i)},\mathbf x^{(j)})$, and thus
$$L_{\theta}(\mathbf x^{(j)})-\kappa \Delta(\mathbf x^{(i)},\mathbf x^{(j)})\leq L_{\theta}(\mathbf x^{(i)})$$
Because $L_{\theta}(\mathbf x^{(i)})\geq 0$ by the assumption (i.e., it is lower bounded by zero), it can be shown that for all $j$
$$\big(L_{\theta}(\mathbf x^{(j)})-\kappa \Delta(\mathbf x^{(i)},\mathbf x^{(j)})\big)_+\leq L_{\theta}(\mathbf x^{(i)}).$$
Hence, by the definition of $\widehat L_{\theta}(\mathbf x)$ and taking the maximum over $j$ on the left hand side, we have
$$\widehat L_{\theta}(\mathbf x^{(i)})\leq L_{\theta}(\mathbf x^{(i)})$$

On the other hand,  we have
$$\widehat L_{\theta}(\mathbf x^{(i)}) \geq L_{\theta}(\mathbf x^{(i)})$$
because $\widehat L_{\theta}(\mathbf x) \geq \big(L_\theta(\mathbf x^{(i)})-\kappa\Delta(\mathbf x,\mathbf x^{(i)})\big)_+$ for any $\mathbf x$, and it is true in particular for $\mathbf x=\mathbf x^{(i)}$.  This shows $\widehat L_{\theta}(\mathbf x^{(i)}) = L_{\theta}(\mathbf x^{(i)})$.

Similarly, one can prove $\widetilde L_\theta(\mathbf x^{(i)})=L_\theta(\mathbf x^{(i)})$. To show this, we have
$$
L_\theta(\mathbf x^{(j)})+\kappa\Delta(\mathbf x^{(i)},\mathbf x^{(j)})\geq L_\theta(\mathbf x^{(i)})
$$
by the Lipschitz continuity of $L_\theta$. By taking the minimum over $j$, we have
$$
\widetilde L_\theta(\mathbf x^{(i)}) \geq L_\theta(\mathbf x^{(i)}).
$$
On the other hand, we have $\widetilde L_\theta(\mathbf x^{(i)})\leq L_\theta(\mathbf x^{(i)})$ by the definition of $\widetilde L_\theta(\mathbf x^{(i)})$.  Combining these two inequalities shows that $\widetilde L_\theta(\mathbf x^{(i)})=L_\theta(\mathbf x^{(i)})$.

Now we can prove 
for any function $L_\theta\in \mathcal F_\kappa$, there exist $\widehat L_\theta$ and $\widetilde L_\theta$ both of which attain the same value of $S_{n,m}$ as $L_\theta$, since $S_{n,m}$ only depends on the values of $L_\theta$ on the data points $\{\mathbf x^{(i)}\}$. In particular, this shows that any global minimum in $\mathcal F_\kappa$ of $S_{n,m}$ can also be attained by the corresponding functions of the form (\ref{eq:param}). By setting $l^*_i=\widehat L_{\theta^*}(\mathbf x^{(i)})=\widetilde L_{\theta^*}(\mathbf x^{(i)})$ for $i=1,\cdots,n+m$, this completes the proof.
\end{proof}

Finally, we prove Corollary~\ref{cor2} that bounds $L_\theta$ with $\widehat L_\theta(\mathbf x)$ and $\widetilde L_\theta(\mathbf x)$ constructed above.
\begin{proof}
By the Lipschitz continuity, we have
$$
L_\theta(\mathbf x^{(i)}) -  \kappa\Delta(\mathbf x,\mathbf x^{(i)}) \leq L_\theta(\mathbf x)
$$
Since $L_\theta(\mathbf x)\geq 0$, it follows that
$$
\big(L_\theta(\mathbf x^{(i)}) -  \kappa\Delta(\mathbf x,\mathbf x^{(i)})\big)_+ \leq L_\theta(\mathbf x)
$$
Taking the maximum over $i$ on the left hand side, we obtain
$$
\widehat L_\theta(\mathbf x) \leq L_\theta(\mathbf x)
$$
This proves the lower bound.

Similarly, we have by Lipschitz continuity
$$
L_\theta(\mathbf x) \leq  \kappa\Delta(\mathbf x,\mathbf x^{(i)}) + L_\theta(\mathbf x^{(i)})
$$
which, by taking the minimum over $i$ on the left hand side, leads to
$$
\widetilde L_\theta(\mathbf x) \leq L_\theta(\mathbf x)
$$
This shows the upper bound.
\end{proof}

\end{document}